%% file: main.tex
\documentclass{article} 

\usepackage[accepted]{icml2019}

\usepackage{dsfont}
\usepackage{amsfonts}
\usepackage{amsmath}
\usepackage{slashbox}

\usepackage{amsthm}

\usepackage{caption}
\usepackage{subcaption}
\usepackage{tabularx,booktabs}
\usepackage{slashbox}
\usepackage{booktabs,caption}
\usepackage[flushleft]{threeparttable}
\usepackage{multirow}
\usepackage{mathtools}
\usepackage{color}

\usepackage{amssymb}
\usepackage{soul}
\usepackage{cancel}
\usepackage{makecell}
\usepackage{multirow}
\usepackage{xcolor}

\usepackage{graphicx}
\usepackage{epstopdf}
\usepackage{tabularx}
\newcolumntype{C}[1]{>{\centering\arraybackslash}p{#1}}

\usepackage[utf8]{inputenc} 
\usepackage[T1]{fontenc}    
\usepackage{hyperref}       
\usepackage{url}            
\usepackage{booktabs}       
\usepackage{amsfonts}       
\usepackage{nicefrac}       
\usepackage{microtype}      

\usepackage{float}
\usepackage{wrapfig}

\usepackage{hyperref}

\newcounter{equationset}
\newcommand{\equationset}[1]{
  \refstepcounter{equationset}
  \noindent\makebox[\linewidth]{Equation set~\theequationset: #1}}

\newtheorem{lemma}{Lemma}[section]

\begin{document}

\twocolumn[
\icmltitle{Exploring Interpretable LSTM Neural Networks over Multi-Variable Data}

\icmlsetsymbol{equal}{*}

\begin{icmlauthorlist}
\icmlauthor{Tian Guo}{eth}
\icmlauthor{Tao Lin}{epfl}
\icmlauthor{Nino Antulov-Fantulin}{eth}
\end{icmlauthorlist}

\icmlaffiliation{eth}{ETH, Z\"urich, Switzerland}
\icmlaffiliation{epfl}{EPFL, Switzerland}

\icmlcorrespondingauthor{Tian Guo}{tian.guo@gess.ethz.ch}

\icmlkeywords{Machine Learning, ICML}
\vskip 0.3in
]

\printAffiliationsAndNotice{} 

\begin{abstract}
For recurrent neural networks trained on time series with target and exogenous variables, in addition to accurate prediction, it is also desired to provide interpretable insights into the data. 
In this paper, we explore the structure of LSTM recurrent neural networks to learn variable-wise hidden states, with the aim to capture different dynamics in multi-variable time series and distinguish the contribution of variables to the prediction.
With these variable-wise hidden states, a mixture attention mechanism is proposed to model the generative process of the target. Then we develop associated training methods to jointly learn network parameters, variable and temporal importance w.r.t the prediction of the target variable.
Extensive experiments on real datasets demonstrate enhanced prediction performance by capturing the dynamics of different variables.
Meanwhile, we evaluate the interpretation results both qualitatively and quantitatively.
It exhibits the prospect as an end-to-end framework for both forecasting and knowledge extraction over multi-variable data.
\end{abstract}
\input{intro}
\input{related}
\input{model}
\input{exp}
\input{conclusion}

\section*{Acknowledgements}
The work has been funded by the EU Horizon 2020 SoBigData project under grant agreement No. 654024.

\bibliography{reference}
\bibliographystyle{icml2019}

\onecolumn
\input{appendix}

\end{document}

%% file: intro.tex
\section{Introduction}

Recently, recurrent neural networks (RNNs), especially long short-term memory (LSTM)~\citep{hochreiter1997long} and gated recurrent units (GRU)~\citep{cho2014properties}, have been proven to be powerful sequence modeling tools in various tasks e.g. language modelling, machine translation, health informatics, time series, and speech~\citep{ke2018focused, lin2017hybrid, guo2016robust, lipton2015learning, sutskever2014sequence, bahdanau2014neural}.
In this paper, we focus on RNNs over multi-variable time series consisting of target and exogenous variables.
RNNs trained over such multi-variable data capture nonlinear correlation of historical values of target and exogenous variables to the future target values. 

In addition to forecasting, interpretable RNNs are desirable for gaining insights into the important part of data for RNNs achieving good prediction performance~\citep{hu2018listening, foerster2017input, lipton2016mythos}.
In this paper, we focus on two types of importance interpretation: variable importance and variable-wise temporal importance.
First, in RNNs variables differ in predictive power on the target, thereby contributing differently to the prediction~\citep{feng2018nonparametric, riemer2016correcting}.
Second, variables also present different temporal relevance to the target one ~\citep{kirchgassner2012introduction}.
For instance, for a variable instantaneously correlated to the target, its short historical data contributes more to the prediction. 
The ability to acquire this knowledge enables additional applications, e.g. variable selection.

However, current RNNs fall short of the aforementioned interpretability for multi-variable data due to their opaque hidden states.
Specifically, when fed with the multi-variable observations of the target and exogenous variables, RNNs blindly blend the information of all variables into the hidden states used for prediction.
It is intractable to distinguish the contribution of individual variables into the prediction through the sequence of hidden states~\citep{zhang2017stock}. 

Meanwhile, individual variables typically present different dynamics. 
This information is implicitly neglected by the hidden states mixing multi-variable data, thereby potentially hindering the prediction performance.

Existing works aiming to enhance the interpretability of recurrent neural networks rarely touch the internal structure of RNNs to overcome the opacity of hidden states on multi-variable data. 
They still fall short of aforementioned two types of interpretation~\citep{montavon2018methods, foerster2017input, che2016interpretable}. 
One category of the approaches is to perform post-analyzing on trained RNNs by perturbation on training data or gradient based methods~\citep{ancona2018towards, ribeiro2018anchors, lundberg2017unified, shrikumar2017learning}.
Another category is to build attention mechanism on hidden states of RNNs to characterize the importance of different time steps~\citep{qin2017dual, choi2016retain}.



In this paper we aim to achieve a unified framework of accurate forecasting and importance interpretation.
In particular, the contribution is fourfold:
\begin{itemize}
\item We explore the structure of LSTM to enable variable-wise hidden states capturing individual variable's dynamics. 
It facilitates the prediction and interpretation.
This family of LSTM is referred to as \textbf{I}nterpretable \textbf{M}ulti-\textbf{V}ariable LSTM, i.e. IMV-LSTM.
\item A novel mixture attention mechanism is designed to summarize variable-wise hidden states and model the generative process of the target.
\item We develop a training method based on probabilistic mixture attention to learn network parameter, variable and temporal importance measures simultaneously.
\item Extensive experimental evaluation of IMV-LSTM against statistical, machine learning and deep learning based baselines demonstrate the superior prediction performance and interpretability of IMV-LSTM.
The idea of IMV-LSTM easily applies to other RNN structures, e.g. GRU and stacked recurrent layers. 
\end{itemize}

%% file: related.tex
\section{Related Work}


Recent research on the interpretable RNNs can be categorized into two groups: attention methods and post-analyzing on trained models.
Attention mechanism has gained tremendous popularity~\citep{xu2018raim, choi2018fine, guo2018interpretable, lai2017modeling, qin2017dual, cinar2017position, choi2016retain, vinyals2015pointer, bahdanau2014neural}.
However, current attention mechanism is mainly applied to hidden states across time steps.
\citet{qin2017dual, choi2016retain} built attention on conventional hidden states of encoder networks. 
Since the hidden states encode information from all input variables, the derived attention is biased when used to measure the importance of corresponding variables. 
The contribution coefficients defined on attention values is biased as well \citep{choi2016retain}. 
Moreover, weighting input data by attentions \citep{xu2018raim, qin2017dual, choi2016retain} does not consider the direction of correlation with the target, which could impair the prediction performance.
Current attention based methods seldom provide variable-wise temporal interpretability.

As for post-analyzing interpretation,~\citet{murdoch2018beyond, murdoch2017automatic, arras2017explaining} extracted temporal importance scores over words or phrases of individual sequences by decomposing the memory cells of trained RNNs.
In perturbation-based approaches perturbed samples might be different from the original data distribution~\citep{ribeiro2018anchors}.
Gradient-based methods analyze the features that output was most sensitive to~\citep{ancona2018towards, shrikumar2017learning}.
Above methods mostly focused on one type of importance and are computationally inefficient. They rarely enhance the predicting performance.

\citet{wavelet} focused on the importance of each middle layer to the output.
\citet{chu2018exact} proposed interpreting solutions for piece-wise linear neural networks.
\citet{foerster2017input} introduced input-switched linear affine transformations into RNNs to analyze the contribution of input steps, wihch could lead to the loss in prediction performance.
Our paper focuses on exploring the internal structure of LSTM so as to learn accurate forecasting and importance measures simultaneously.

Another line of related research is about tensorization and decomposition of hidden states in RNNs. 
\citet{do2017matrix, novikov2015tensorizing} proposed to represent hidden states as matrices. 
\citet{he2017wider} developed tensorized LSTM to enhance the capacity of networks without additional parameters.
\citet{kuchaiev2017factorization, neil2016phased, koutnik2014clockwork} proposed to partition the hidden layer into separated modules with different updates. 
These hidden state tensors and update processes 
do not maintain variable-wise correspondence and lack the desirable interpretability.

%% file: model.tex
\section{Interpretable Multi-Variable LSTM}\label{sec:mv}
In the following we will first explore the internal structure of LSTM to enable hidden states to encode individual variables, such that the contribution from individual variables to the prediction can be distinguished.
Then, mixture attention is designed to summarize these variable-wise hidden states for predicting. 
The described method can be easily extended to multi-step ahead prediction via iterative methods as well as vector regression~\citep{deepmulti2018, cheng2006multistep}.
\begin{figure*}[htbp!]
\vskip 0.2in
\begin{center}
\centerline{\includegraphics[width = 0.9\textwidth]{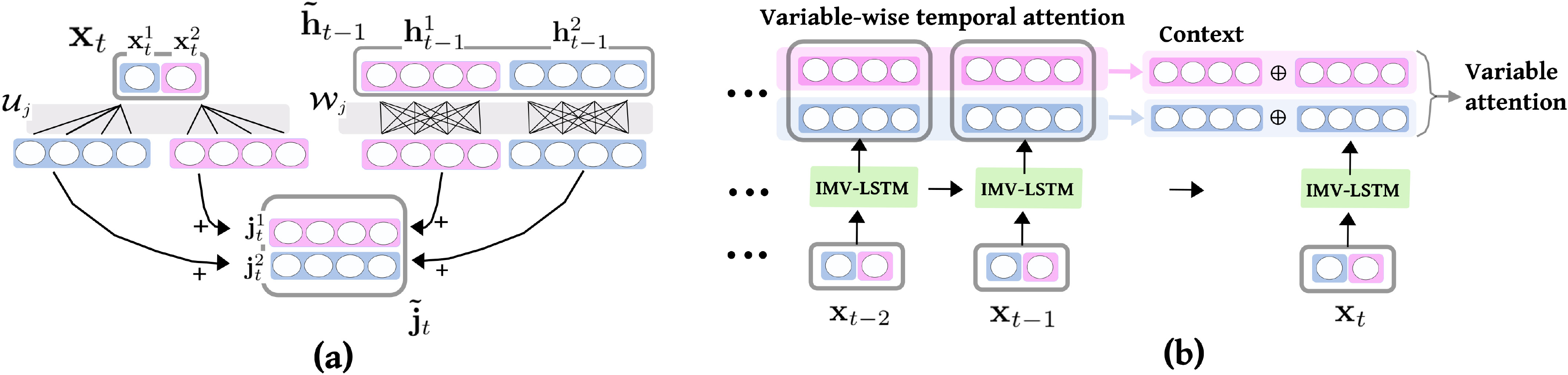}}
\caption{A toy example of a IMV-LSTM with a two-variable input sequence and the hidden matrix of  $4$-dimensions per variable.
Circles represent one dimensional elements.
Purple and blue colors correspond to two variables. 
Blocks containing rectangles with circles inside represent input data and hidden matrix.
Panel (a) exhibits the derivation of hidden update $\mathbf{\tilde{j}}_t$. 
Grey areas represent transition weights. 
Panel (b) demonstrates the mixture attention process. (best viewed in color)}
\label{fig:graph}
\end{center}
\vskip -0.2in
\end{figure*}


Assume we have $N$-$1$ exogenous time series and a target series $\mathbf{y}$ of length $T$, where $\mathbf{y} = [ y_1, \cdots, y_T ]$ and $\mathbf{y} \in \mathbb{R}^T$. 
{Vectors are assumed to be in column form throughout this paper.}
By stacking exogenous time series and target series, we define a multi-variable input series as
$\mathbf{X}_{T} = \{ \mathbf{x}_1, \cdots, \mathbf{x}_T  \}$, 
where $\mathbf{x}_t = [ \mathbf{x}_{t}^1, \cdots, \mathbf{x}_{t}^{N-1}, {y}_{t} ]$.
Both of $\mathbf{x}_{t}^{n}$ and ${y}_{t}$ can be multi-dimensional vector.
$\mathbf{x}_t \in \mathbb{R}^N$ is the multi-variable input at time step $t$.
It is also free for $\mathbf{X}_T$ to merely include exogenous variables, which does not affect the methods presented below. 

Given $\mathbf{X}_T$, we aim to learn a non-linear mapping to predict the next values of the target series, namely $\hat{y}_{T+1} = \mathcal{F}(\mathbf{X}_T)$.

Meanwhile, the other desirable byproduct of learning $\mathcal{F}(\mathbf{X}_T)$ is the variable and temporal importance measures.
Mathematically, we aim to derive variable importance vector $\mathbf{I} \in R^{N}_{\geq 0}$, $\sum_{n=1}^N \text{I}_{n} = 1$ and variable-wise temporal importance vector $\mathbf{T}^n \in R^{T-1}_{\geq 0}$ (w.r.t. variable $n$), $\sum_{k = 1}^{T-1} \text{T}^n_{k} = 1$.
Elements of these vectors are normalized (i.e. sum to one) and reflect the relative importance of the corresponding variable or time instant w.r.t. the prediction.

\subsection{Network Architecture}
The idea of IMV-LSTM is to make use of hidden state matrix and to develop associated update scheme, such that each element (e.g. row) of the hidden matrix encapsulates information exclusively from a certain variable of the input.

To distinguish from the hidden state and gate vectors in a standard LSTM, hidden state and gate matrices in IMV-LSTM are denoted with tildes.
Specifically, we define the hidden state matrix at time step $t$ as $\mathbf{\tilde{h}}_t = [ \mathbf{h}_t^{1} \, , \cdots, \, \mathbf{h}_t^{N} ]^\top$, where $\mathbf{\tilde{h}}_t \in \mathbb{R}^{N \times d}$, $\mathbf{h}_t^{n} \in \mathbb{R}^{d}$.
The overall size of the layer is derived as $D = N \cdot d$. 
The element $\mathbf{h}_t^{n}$ of $\mathbf{\tilde{h}}_t$ is the hidden state vector specific to $n$-th input variable.

Then, we define the input-to-hidden transition as $\boldsymbol{\mathcal{U}}_j = [\mathbf{U}^{1}_j \, , \cdots, \, \mathbf{U}^{N}_j ]^\top$, 
where $\boldsymbol{\mathcal{U}}_j \in \mathbb{R}^{N \times d \times d_0 }$, $\mathbf{U}^{n}_j \in \mathbb{R}^{d \times d_0} $ and $d_0$ is the dimension of individual variables at each time step.
The hidden-to-hidden transition is defined as: 
$\boldsymbol{\mathcal{W}}_j = [ \mathbf{W}_j^{1} \, \cdots \, \mathbf{W}_j^{N} ]$, 
where $\boldsymbol{\mathcal{W}}_j \in \mathbb{R}^{N \times d \times d}$ and 
$\mathbf{W}_j^{n} \in \mathbb{R}^{d \times d}$.

As standard LSTM neural networks~\citep{hochreiter1997long}, IMV-LSTM has the input $\mathbf{i}_t$, forget $\mathbf{f}_t$, output gates $\mathbf{o}_t$ and the memory cells $\mathbf{c}_t$ in the update process.
Given the newly incoming input $\mathbf{x}_t$ at time $t$ and the hidden state matrix $\mathbf{\tilde{h}}_{t-1}$, the hidden state update is defined as:
\begin{equation}
\mathbf{\tilde{j}}_t = \tanh \left( \boldsymbol{\mathcal{W}}_j \circledast \mathbf{\tilde{h}}_{t-1} + \boldsymbol{\mathcal{U}}_j \circledast \mathbf{x}_t + \mathbf{b}_j \right) \,,  \label{eq:j}
\end{equation}
where $\mathbf{\tilde{j}}_t = [ \, \mathbf{j}_t^1, \cdots, \mathbf{j}_t^N \, ]^{\top}$ has the same shape as hidden state matrix $\mathbb{R}^{N \times d}$. 
Each element $ \mathbf{j}_t^n \in \mathbb{R}^{d}$ corresponds to the update of the hidden state
w.r.t. input variable $n$. Term $\boldsymbol{\mathcal{W}}_j \circledast \mathbf{\tilde{h}}_{t-1}$ and $\boldsymbol{\mathcal{U}}_j \circledast \mathbf{x}_t $ respectively capture the update from the hidden states at the previous step and the new input.
The tensor-dot operation $\circledast$ is defined as the product of two tensors along the $N$ axis,
e.g., $\boldsymbol{\mathcal{W}}_j \circledast \mathbf{\tilde{h}}_{t-1} = [ \mathbf{W}_j^1 \mathbf{h}_{t-1}^1 \, , \cdots, \, \mathbf{W}_j^N \mathbf{h}_{t-1}^N ]^\top $ where $ \mathbf{W}_j^n \mathbf{h}_{t-1}^n  \in \mathbb{R}^{d}$.


Depending on different update schemes of gates and memory cells,  we proposed two realizations of IMV-LSTM, i.e. IMV-Full in Equation set \ref{eqset:1} and IMV-Tensor in Equation set \ref{eqset:2}.
In these two sets of equations, $\text{vec}(\cdot)$ refers to the vectorization operation, which concatenates columns of a matrix into a vector.
The concatenation operation is denoted by $\oplus$ and element-wise multiplication is denoted by $\odot$.
Operator $\text{matricization}(\cdot)$ reshapes a vector of $\mathbb{R}^D$ into a matrix of $\mathbb{R}^{N \times d}$.
\begin{minipage}{0.47\textwidth}
\begin{align}
&\begin{bmatrix}
        \mathbf{i}_t \\
        \mathbf{f}_t \\
        \mathbf{o}_t 
        \end{bmatrix} =  \sigma \left( \mathbf{W} \text{ } [\mathbf{x}_t \oplus \text{vec}(\text{ } \mathbf{\tilde{h}}_{t-1})] + \mathbf{b} \right) \label{eq:full_gate}\\
&\text{ }\mathbf{c}_t = \mathbf{f}_t \odot \mathbf{c}_{t-1} + \mathbf{i}_t \odot \text{vec}( \text{ }\mathbf{\tilde{j}}_t) \label{eq:full_c}\\
&\text{ }\mathbf{\tilde{h}}_t = \text{matricization}( \mathbf{o}_t \odot \tanh(\mathbf{c}_t) ) \label{eq:full_h}
\end{align}
\equationset{IMV-Full}\label{eqset:1}
\end{minipage}%
\hfill
\begin{minipage}{0.47\textwidth}
\begin{align}
&\begin{bmatrix}
        \mathbf{\tilde{i}}_t \\
        \mathbf{\tilde{f}}_t \\
        \mathbf{\tilde{o}}_t 
        \end{bmatrix} =  \sigma \left( \boldsymbol{\mathcal{W}} \circledast \mathbf{\tilde{h}}_{t-1} + \boldsymbol{\mathcal{U}} \circledast \mathbf{x}_t + \mathbf{b} \right) \label{eq:tensor_gate}\\
&\text{ }\mathbf{\tilde{c}}_t = \mathbf{\tilde{f}}_t \odot \mathbf{\tilde{c}}_{t-1} + \mathbf{\tilde{i}}_t \odot \mathbf{\tilde{j}}_t \label{eq:tensor_c}\\
&\text{ }\mathbf{\tilde{h}}_t = \mathbf{\tilde{o}}_t \odot \tanh(\mathbf{\tilde{c}}_t) \label{eq:tensor_h}
\end{align}
\equationset{IMV-Tensor}\label{eqset:2}
\end{minipage}

\textbf{IMV-Full}: 
With vectorization in Eq.~\eqref{eq:full_gate} and \eqref{eq:full_c}, IMV-Full updates gates and memories using full $\mathbf{\tilde{h}_{t-1}}$ and $\mathbf{\tilde{j}_t}$ regardless of the variable-wise data in them. 
By simple replacement of the hidden update in standard LSTM by $\mathbf{\tilde{j}}_t$, IMV-Full behaves identically to standard LSTM while enjoying the interpretability shown below.

\textbf{IMV-Tensor}:
By applying tensor-dot operations in Eq.~\eqref{eq:tensor_gate}, gates and memory cells are matrices as well, elements of which have the correspondence to input variables as hidden state matrix $\mathbf{\tilde{h}}_t$ does.
$\boldsymbol{\mathcal{W}}$ and $\boldsymbol{\mathcal{U}}$ have the same shapes as $\boldsymbol{\mathcal{W}_j}$ and $\boldsymbol{\mathcal{U}_j}$ in Eq.~\eqref{eq:j}.

In IMV-Full and IMV-Tensor, gates only scale $\mathbf{\tilde{j}_t}$ and $\mathbf{\tilde{c}}_{t-1}$ and thus retain the variable-wise data organization in $\mathbf{\tilde{h}}_t$. 
Meanwhile, based on tensorized hidden state Eq.~\eqref{eq:j} and gate update Eq.~\eqref{eq:tensor_gate}, IMV-Tensor can also be considered as a set of parallel LSTMs, each of which processes one variable series and then merges via the mixture. 
The derived hidden states specific to each variable are aggregated on both temporal and variable level through the mixture attention.

Next, we provide the analysis about the complexity of IMV-LSTM through Lemma~\ref{lemma:size} and Lemma~\ref{lemma:time}.

\begin{lemma}\label{lemma:size}
	Given time series of $N$ variables, assume a standard LSTM and IMV-LSTM layer both have size $D$, i.e. $D$ neurons in the layer.
	Then, compared to the number of parameters of the standard LSTM, IMV-Full and IMV-Tensor respectively reduce the network complexity by $(N-1)D + (1-1/N)D \cdot D$ and $4(N-1)D + 4(1-1/N)D \cdot D$ number of parameters. 
\end{lemma}
\begin{proof}
	In a standard LSTM of layer size $D$, trainable parameters lie in the hidden and gate update functions.
	In total, these update functions have $4 D\cdot D + 4 N \cdot D + 4 D$ parameters, where $4 D\cdot D + 4 N \cdot D$ comes from the transition and $4D$ corresponds to the bias terms.	
	For IMV-Full, assume each input variable corresponds to one-dimensional time series. 
	Based on Eq.~\ref{eq:j}, the hidden update has $2D + D^2/N$ trainable parameters. 
	Equation set \ref{eqset:1} gives rise to the number of parameters equal to that of the standard LSTM.
	Therefore, the reduce number of parameters is $(N-1)D + (1-1/N)D \cdot D$.
	As for IMV-Tensor, more parameter reduction stems from that the gate update functions in Equation set \ref{eqset:2} make use of the tensor-dot operation as Eq.~\ref{eq:j}. 
\end{proof}

\begin{lemma}\label{lemma:time}
	For time series of $N$ variables and the recurrent layer of size $D$, IMV-Full and IMV-Tensor respectively have the computation complexity at each update step as: $\mathcal{O}(D^2 + N\cdot D)$ and $\mathcal{O}(D^2/N + D)$.
\end{lemma}
\begin{proof}
Assume that $D$ neurons of the recurrent layer in IMV-Full and IMV-Tensor are evenly assigned to $N$ input variables, namely each input variable has $d = D/N$ corresponding neurons. 
For IMV-Full, based on Eq.~\ref{eq:j}, the hidden update has computation complexity $N \cdot d^2 + N \cdot d$, while the gate update process has the complexity $D^2 + N \cdot D$. Overall, the computation complexity is $\mathcal{O}(D^2 + N\cdot D)$, which is identical to the complexity of a standard LSTM.
As for IMV-Tensor, since the gate update functions in Equation set \ref{eqset:2} make use of the tensor-dot operation as Eq.~\ref{eq:j}, gate update functions have the same computation complexity as Eq.~\ref{eq:j}. The overall complexity is $\mathcal{O}(D^2/N + D)$, which is $1/N$ of the complexity of a standard LSTM. 
\end{proof}

Basically, Lemma~\ref{lemma:size} and Lemma~\ref{lemma:time} indicate that a high number of input variables leads to a large portion of parameter and computation reduction in IMV-LSTM family.

\subsection{Mixture Attention}
After feeding a sequence of $\{ \mathbf{x}_1, \cdots, \mathbf{x}_T \}$ into IMV-Full or IMV-Tensor, we obtain a sequence of hidden state matrices $\{ \mathbf{\tilde{h}}_1, \cdots, \mathbf{\tilde{h}}_{T} \}$, where the sequence of hidden states specific to variable $n$ is extracted as $\{ \mathbf{h}_1^n, \cdots, \mathbf{h}_T^n \}$. 

The idea of mixture attention mechanism as follows.
Temporal attention is first applied to the sequence of hidden states corresponding to each variable, so as to obtain the summarized history of each variable. 
Then by using the history enriched hidden state of each variable, variable attention is derived to merge variable-wise states. 
These two steps are assembled into a probabilistic mixture model~\citep{zong2018deep, graves2013generating, bishop1994mixture}, which facilitates the subsequent learning, predicting, and interpreting.  

In particular, the mixture attention is formulated as:
\begin{align}\label{eq:mix}
& p(y_{T+1} \, | \mathbf{X}_{T}) \nonumber  \\ 
& = \sum_{n=1}^{N} p(y_{T+1} | z_{T+1} = n, \mathbf{X}_{T} ) \cdot \Pr(z_{T+1} = n | \mathbf{X}_{T}) \nonumber \\
& =  \sum_{n=1}^{N} p(y_{T+1} \, | \, z_{T+1} = n, \mathbf{h}_1^n, \cdots, \mathbf{h}_T^n ) \nonumber \\ 
&\quad \cdot \Pr(z_{T+1} = n \, | \, \mathbf{\tilde{h}}_1, \cdots, \mathbf{\tilde{h}}_{T}) \nonumber \\
& =  \sum_{n=1}^{N} p( y_{T+1} \, | \, z_{T+1} = n, \underbrace{\mathbf{h}_T^n \oplus \mathbf{g}^n}_{ {\substack{\text{variable-wise} \\ \text{temporal attention}}}} ) \\ 
&\quad \cdot \underbrace{ \Pr(z_{T+1} = n \, | \, \mathbf{h}_T^1 \oplus \mathbf{g}^1, \cdots, \mathbf{h}_T^N \oplus \mathbf{g}^N )}_{\text{Variable attention}} \nonumber  
\end{align}
In Eq. \eqref{eq:mix}, we introduce a latent random variable $z_{T+1}$ into the the density function of $y_{T+1}$ to govern the generation process.
$z_{T+1}$ is a discrete variable over the set of values $\{ 1, \cdots, N \}$ corresponding to $N$ input variables.
Mathematically, $p( y_{T+1} \, | \, z_{T+1} = n, \mathbf{h}_T^n \oplus \mathbf{g}^n )$ characterizes the density of $y_{T+1}$ conditioned on historical data of variable $n$, while the prior of $z_{T+1}$, i.e. $\Pr(z_{T+1} = n \, | \, \mathbf{h}_T^1 \oplus \mathbf{g}^1, \cdots, \mathbf{h}_T^N \oplus \mathbf{g}^N )$ controls to what extent $y_{T+1}$ is driven by variable $n$.

Context vector $\mathbf{g}^n$ is computed as the temporal attention weighted sum of hidden states of variable $n$, i.e., $\mathbf{g}^n = \sum_t \alpha_t^n \mathbf{h}_t^n$. The attention weight $\alpha_t^n$ is evaluated as $\alpha_t^n = \frac{\exp{( \, \text{f}_n(\mathbf{h}_t^n)} \,)}{ \sum_k \exp{( \, \text{f}_n(\mathbf{h}_k^n)} \,) }$,
where $\text{f}_n(\cdot)$ can be a flexible function specific to variable $n$, e.g. neural networks.

For $p( y_{T+1} \, | \, z_{T+1} = n, \mathbf{h}_T^n \oplus \mathbf{g}^n )$, 
without loss of generality, we use a Gaussian output distribution parameterized by $[\, \mu_n, \sigma_n] = \varphi_n(\, \mathbf{h}_T^n \oplus \mathbf{g}^n )$, where $\varphi_n(\cdot)$ can be a feed-forward neural network.
It is free to use other distributions.


$\Pr(z_{T+1} = n \, | \, \mathbf{h}_T^1 \oplus \mathbf{g}^1, \cdots, \mathbf{h}_T^N \oplus \mathbf{g}^N )$ is derived by a softmax function over $\{ \text{f}( \, \mathbf{h}_T^n \oplus \mathbf{g}^n ) \}_N$, where $\text{f}(\cdot)$ can be a feedforward neural network shared by all variables. 


\subsection{Learning to Interpret and Predict}
In the learning phase, the set of parameters in the neural network and mixture attention is denoted by $\Theta$.
Given a set of $M$ training sequences $\{\mathbf{X}_{T}\}_M$ and $\{y_{T+1}\}_M$, we aim to learn both $\Theta$ and importance vectors $\mathbf{I}$ and $\{\mathbf{T}^n\}_N$ for prediction and insights into the data.

Next, we first illustrate the burden of directly interpreting attention values and then present the training method combining parameter and importance vector learning, without the need of post analyzing. 

Importance vectors $\mathbf{I}$ and $\{\mathbf{T}^n\}_N$ reflect the global relations in variables, while the attention values derived above are specific to data instances.
Moreover, it is nontrivial to decipher variable and temporal importance from attentions.

For instance, during the training on PLANT dataset used in the experiment section, we collect variable and variable-wise temporal attention values of training instances.
In Fig~\ref{fig:att_interpret}, left panels plots the histograms of variable attention of three variables in PLANT at two different epochs.
It is difficult to fully discriminate variable importance from these histograms.
Likewise, in the right panel, histograms of temporal attentions at certain time lags of variable ``P-temperature'' at two different epochs does not ease the importance interpretation.
Time lag represents the look-back time step w.r..t the current one.
Similar phenomena are observed in other variables and datasets during the experiments.

In the following, we develop the training procedure based on the Expectation–Maximization (EM) framework for the probabilistic model with latent variables, i.e. Eq.~\eqref{eq:mix} in this paper. 
Index $m$ in the following corresponds to the training data instance. 
It is omitted in $\mathbf{h}_T^n$ and $\mathbf{g}^n$ for simplicity.

The loss function to minimize is derived as:
\begin{equation}\label{eq:loss}
\textstyle
\begin{split}
& \mathcal{L}(\Theta, \mathbf{I} ) = \\
&\quad - \sum_{m=1}^{M}  \mathbb{E}_{q_m^n}[ \, \log p( y_{T+1, m} \, | \, z_{T+1, m} = n, \mathbf{h}_T^n \oplus \mathbf{g}^n ) ] \\
&\quad - \mathbb{E}_{q_m^n}[ \, \log \Pr(z_{T+1, m} = n \, | \, \mathbf{h}_T^1 \oplus \mathbf{g}^1, \cdots, \mathbf{h}_T^N \oplus \mathbf{g}^N )] \\
&\quad - \mathbb{E}_{q_m^n}[ \, \log \Pr(z_{T+1, m} = n \, | \,  \mathbf{I} )
]
\end{split}
\end{equation}
The desirable property of this loss function is as:
\begin{lemma}\label{lemma:bound}
The negative log-likelihood defined by Eq.~\eqref{eq:mix} is upper-bounded by the loss function Eq.~\eqref{eq:loss} in the EM process: 

\centering$-\log\prod\limits_{m} p(y_{T+1, m} \,|\mathbf{X}_{T,m} \,; \Theta) \leq \mathcal{L}(\Theta, \mathbf{I})$
\end{lemma}
The proof is provided in the supplementary material. 

 \begin{figure}[htbp]
     \vspace{-0pt}
     \centering
     \begin{subfigure}[t]{0.23\textwidth}
         \centering
         \includegraphics[width=0.95\textwidth]{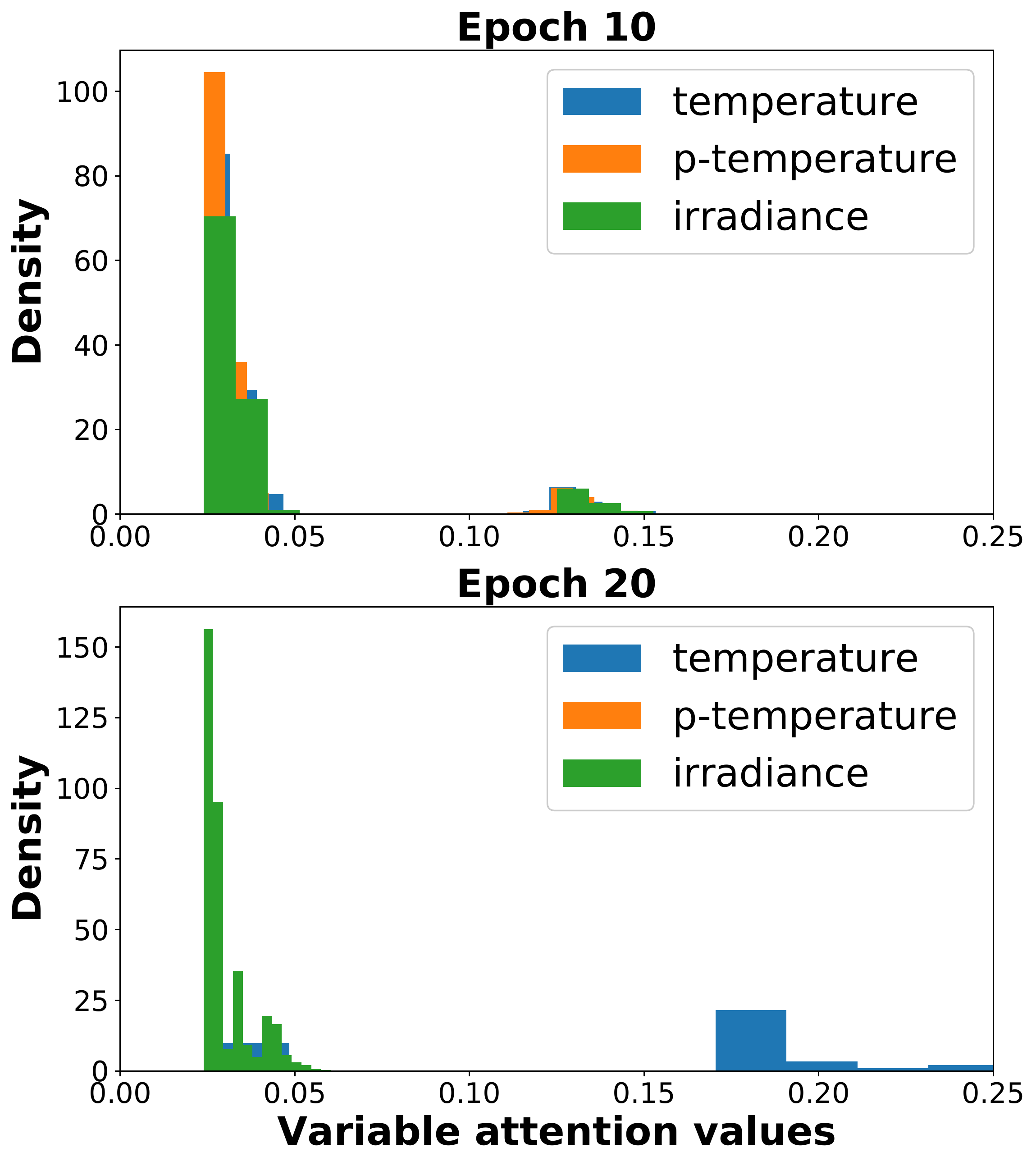}
     \end{subfigure}
     \begin{subfigure}[t]{0.24\textwidth}
         \centering
         \includegraphics[width=0.95\textwidth]{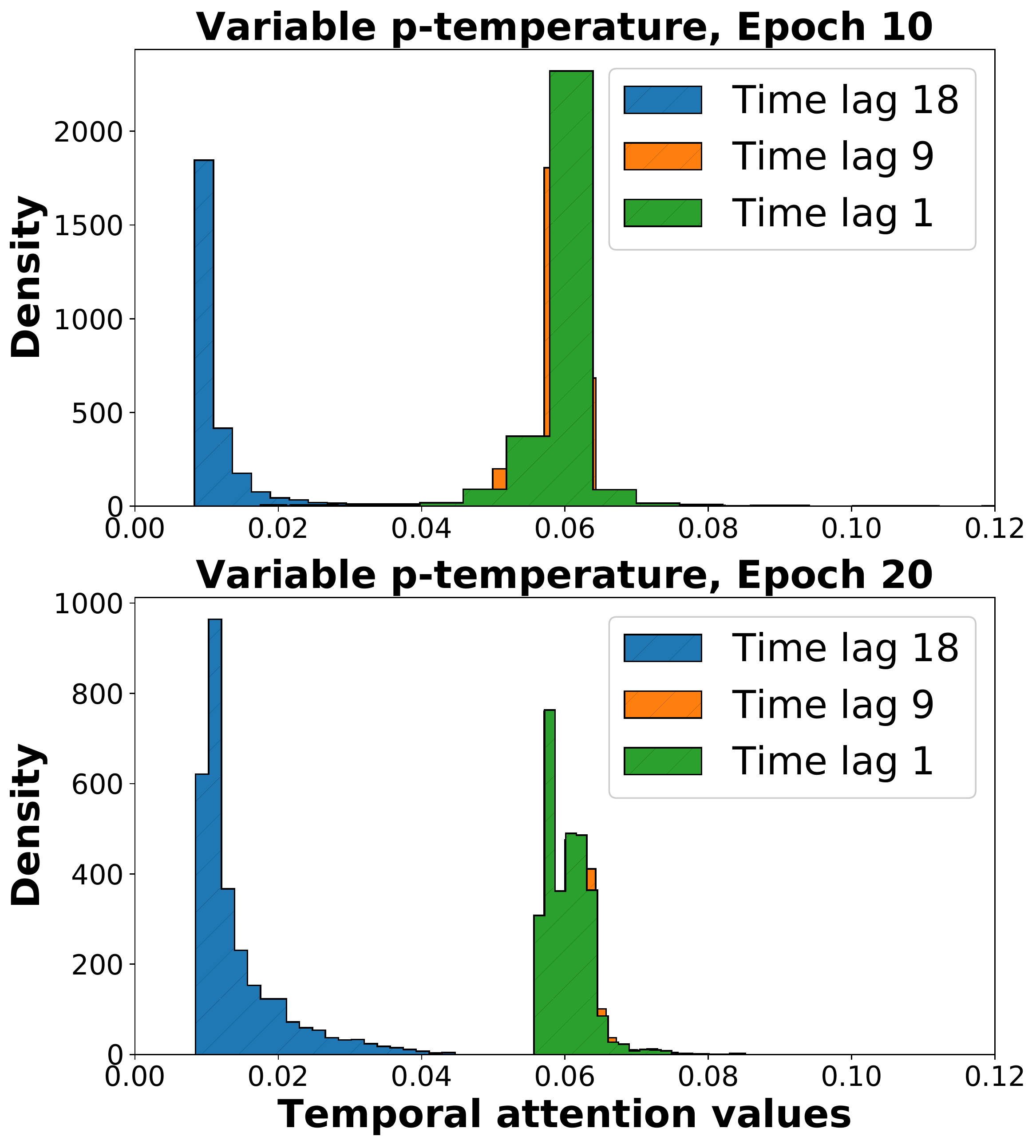}
     \end{subfigure}
     \caption{\small{Left: Histograms of variable attention at different epochs. Right: Histograms of temporal attention of variable ``P-temperature''. It is nontrivial to interpret variable and temporal importance from attention values on training instances.}
     }
     \vspace{-5pt}
     \label{fig:att_interpret}
 \end{figure}
Therefore, minimizing Eq.~\eqref{eq:loss} enables to simultaneously learn the network parameters and importance vectors without the need of post processing on trained networks.


In particular, in Eq.~\eqref{eq:loss} the first two terms
are derived from the standard EM procedure.
For the last term $\mathbb{E}_{q_m^n}[ \, \log \Pr(z_{T+1, m} = n \, | \,  \mathbf{I} )$, intuitively it serves as a regularization on the posterior of latent variable $z_{T+1, m}$ and encourages individual instances to follow the global pattern $\mathbf{I}$, which parameterizes a discrete distribution on $z_{T+1, m}$. 
And $q^n_m$ represents the posterior of $z_{T+1, m}$ by Eq.~\eqref{eq:poster}:
\begin{equation}\label{eq:poster}
\textstyle
\begin{split}
\textstyle
& q^n_m  \coloneqq \Pr (z_{T+1, m}  = n | \, \mathbf{X}_{T, m}, y_{T+1, m} \, ; \Theta) \\
& \propto p(y_{T+1, m} | z_{T+1, m} = n, \mathbf{X}_{T, m} ) \cdot \Pr(z_{T+1, m} = n | \mathbf{X}_{T, m}) \\
& \approx p( y_{T+1, m} \, | \, z_{T+1, m} = n,  \mathbf{h}_T^n \oplus \mathbf{g}^n  ) \\
&\quad \cdot \Pr(z_{T+1, m} = n \, | \, \mathbf{h}_T^1 \oplus \mathbf{g}^1, \cdots, \mathbf{h}_T^N \oplus \mathbf{g}^N )
\end{split}
\end{equation}
During the training phase, network parameters $\Theta$ and importance vectors are alternatively learned. 
In a certain round of the loss function minimization, we first fix the current value of $\Theta$ and evaluate $q_m^n$ for the batch of data. 
Then, since the first two terms in the loss functions solely depend on network parameters $\Theta$, they are minimized via gradient descent to update $\Theta$. 
For the last term, fortunately we can derive a simple closed-form solution of $\mathbf{I}$ as:
\begin{equation}
\textstyle
\mathbf{I} = \frac{1}{M} \sum_m \mathbf{q}_m,  \text{    } \mathbf{q}_m = [q_m^1, \cdots, q_m^n]^{\top}
\end{equation}
, which takes into account both variable attention and predictive likelihood in the importance vector.

As for temporal importance, it can also be derived through EM, but it requires a hierarchical mixture. 
For the sake of computing efficiency, the variable-wise temporal importance vector is derived from attention values as follows:
\begin{equation}
\textstyle
\mathbf{T}^n =  \frac{1}{M} \sum_m \boldsymbol{\alpha}_m^n,  \text{    } \boldsymbol{\alpha}_m = [\alpha_{1, m}^n, \cdots, \alpha_{T, m}^n]^{\top}
\end{equation}
This process iterates until convergence.
After the training, we obtain the neural networks ready for predicting as well as the variable and temporal importance vectors.
Then, in the predicting phase, the prediction of $y_{T+1}$ is obtained by the weighted sum of means as:
\begin{small}
\begin{equation}
\textstyle
\hat{y}_{T+1} = \sum_n  \mu_n \cdot \Pr(z_{T+1} = n \, | \, \mathbf{h}_T^1 \oplus \mathbf{g}^1, \cdots, \mathbf{h}_T^N \oplus \mathbf{g}^N ) \,.
\end{equation}
\end{small}
\vspace{-0.8cm}

%% file: exp.tex
\section{Experiments}
\begin{table*}[htbp]
  \centering
  \caption{\small{RMSE and MAE with std. errors}}
  \small
 \resizebox{1.0\textwidth}{!}{%
  \begin{tabular}{|C{1.55cm}|C{4.0cm}|C{4.4cm}|c|}
    \hline
     {Dataset} &  PM2.5 & PLANT & SML  \\
    \hline
    STRX &     $52.51 \pm 0.82$, $47.35 \pm 0.92$  &  $231.43 \pm 0.19$, $193.23 \pm 0.43$ & $0.039 \pm 0.001$, $0.033 \pm 0.001$ \\
    ARIMAX &   $42.51 \pm 1.13$, $40.23 \pm 0.83$  &  $225.54 \pm 0.23$, $193.42 \pm 0.41$ & $0.060 \pm 0.002$, $0.053 \pm 0.002$ \\
    \hline
    RF &       $38.84 \pm 1.12$, $22.27 \pm 0.63$  &  $164.23 \pm 0.65$, $130.90 \pm 0.15$ & $0.045 \pm 0.001$, $0.032 \pm 0.001$ \\
    XGT &      $25.28 \pm 1.01$, $15.93 \pm 0.72$  &  $164.10 \pm 0.54$, $131.47 \pm 0.21$ & $0.017 \pm 0.001$, $0.013 \pm 0.001$ \\
    ENET &     $26.31 \pm 1.33$, $15.91 \pm 0.51$  &  $168.22 \pm 0.49$, $137.04 \pm 0.38$ & $0.018 \pm 0.001$, $0.015 \pm 0.001$ \\
    \hline
    DUAL &     $25.31 \pm 0.91$, $16.21 \pm 0.42$  & $163.29 \pm 0.54$, $130.87 \pm 0.12$ & $0.019 \pm 0.001$, $0.015 \pm 0.001$ \\
    RETAIN &   $31.12 \pm 0.97$, $20.11 \pm 0.76$   & $250.69 \pm 0.36$, $190.11 \pm 0.15$ & $0.048 \pm 0.001$, $0.037 \pm 0.001$ \\
    \hline
    IMV-Full & $24.47 \pm 0.34$, $15.23 \pm 0.61$  &  $157.32 \pm 0.21$, $128.42 \pm 0.15$  &  $0.015 \pm 0.002$, $0.012 \pm 0.001$ \\
    IMV-Tensor &  $\mathbf{24.29 \pm 0.45}$, $\mathbf{14.87 \pm 0.44}$  &  $\mathbf{156.32 \pm 0.31}$, $\mathbf{127.42 \pm 0.21}$  & $\mathbf{0.009 \pm 0.0009}$, $\mathbf{0.006 \pm 0.0005}$ \\
    \hline
\end{tabular}%
}
\vspace{-10pt}
\label{tab:full}
\end{table*}

\subsection{Datasets}
\textbf{PM2.5:}
It contains hourly PM2.5 data and the associated meteorological data in Beijing of China. PM2.5 measurement is the target series. The exogenous time series include dew point, temperature, pressure, combined wind direction, accumulated wind speed, hours of snow, and hours of rain. Totally we have $41,700$ multi-variable sequences.


\textbf{PLANT:}
This records the time series of energy production of a photo-voltaic power plant in Italy \citep{ceci2017predictive}. Exogenous data consists of 9 weather conditions variables (such as temperature, cloud coverage, etc.). The power production is the target. 
It provides 20842 sequences split into training ($70\%$), validation ($10\%$) and testing sets ($20\%$).

\textbf{SML} is a public dataset used for indoor temperature forecasting. 
Same as \citep{qin2017dual}, the room temperature is taken as the target series and another 16 time series are exogenous series.
The data were sampled every minute.
The first 3200, the following 400 and the last 537 data points are respectively used for training, validation, and test.

Due to the page limitation, experimental results on additional datasets are in the supplementary material.
\begin{figure*}[htbp]
    \vskip 0.2in
    \centering
    \begin{subfigure}[t]{0.32
    \textwidth}
        \centering
        \includegraphics[width=0.99\textwidth]{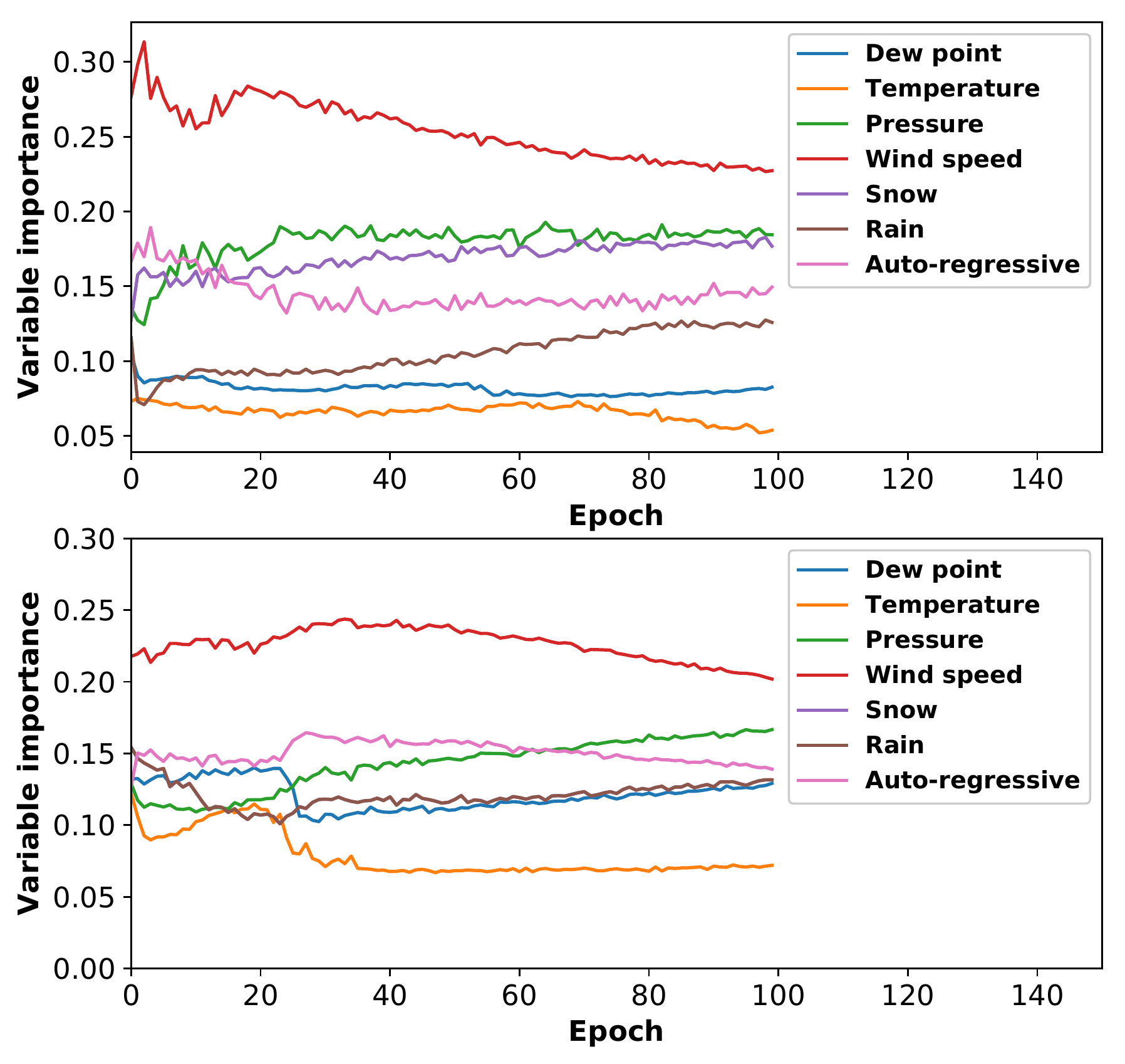}
        \caption{\small{PM2.5}}
    \end{subfigure}%
    ~
    \begin{subfigure}[t]{0.33\textwidth}
        \centering
        \includegraphics[width=0.99\textwidth]{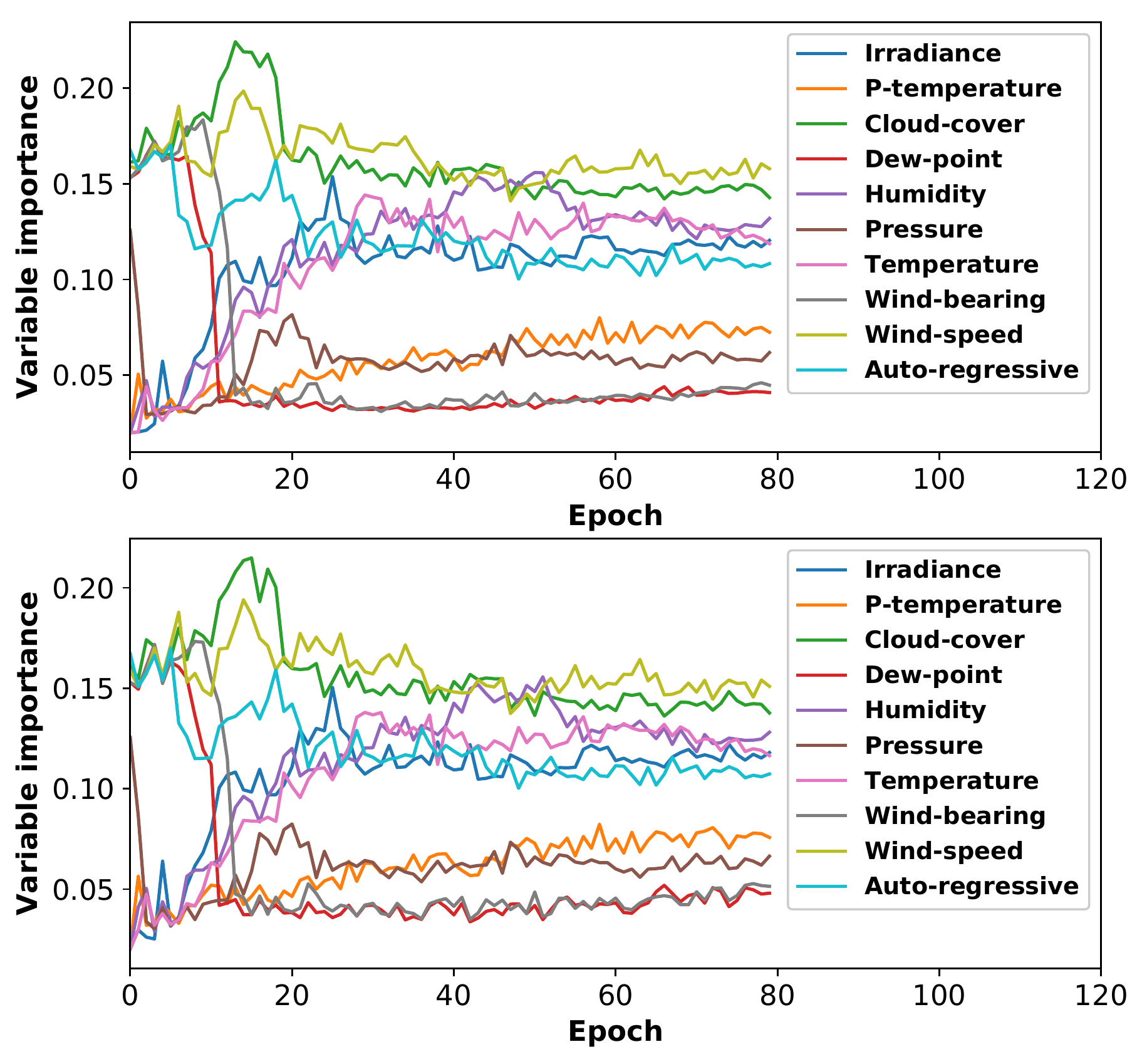}
        \caption{\small{PLANT}}
    \end{subfigure}%
    ~
    \begin{subfigure}[t]{0.32\textwidth}
        \centering
        \includegraphics[width=0.99\textwidth]{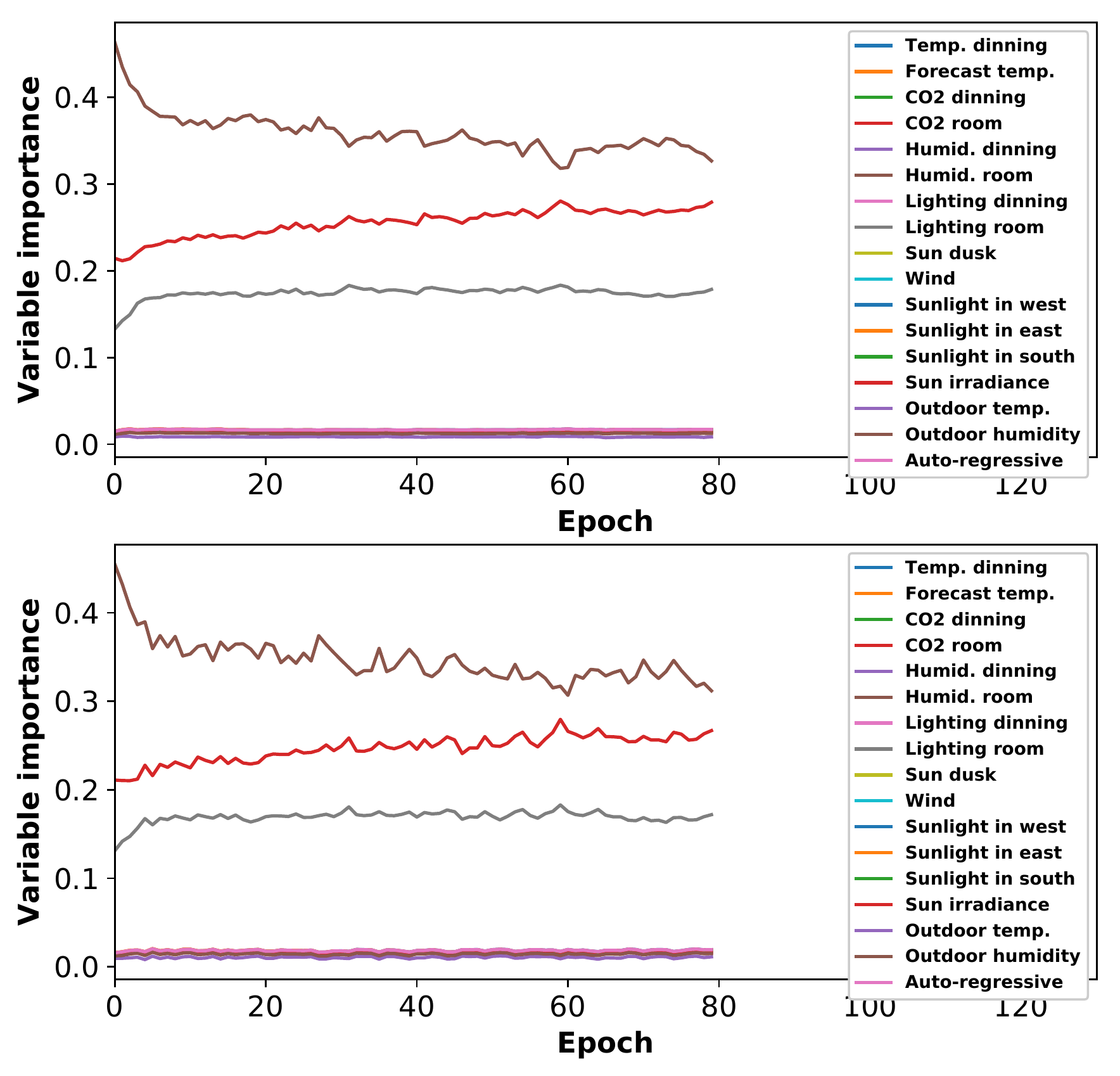}
        \caption{\small{SML}}
    \end{subfigure}%
    ~
    \vspace{-10pt}
    \caption{\small Variable importance over epochs during the training. Top and bottom panels in each sub-figure respectively correspond to IMV-Full and IMV-Tensor. 
    In Sec. 4.4., we provide evidence from domain knowledge and show the agreement with the discoveries by IMV-Full and IMV-Tensor.
}
    \label{fig:impt_var}
    \vskip -0.2in
\end{figure*}

\subsection{Baselines and Evaluation Setup}
The first category of statistics baselines includes:

\textbf{STRX} is the structural time series model with exogenous variables 
~\citep{scott2014predicting, radinsky2012modeling}. 
It is consisted of unobserved components via state space models.

\textbf{ARIMAX} is the auto-regressive integrated moving average with regression terms on exogenous variables~\citep{hyndman2014forecasting}. 
It is a special case of vector auto-regression in this scenario.

The second category of  machine learning baselines includes: 

\textbf{RF} refers to random forests, an ensemble learning method consisting of several decision trees~\citep{liaw2002classification} and was used in time series prediction~\citep{patel2015predicting}.

\textbf{XGT} refers to the extreme gradient boosting~\citep{chen2016xgboost}. 
It is the application of boosting methods to regression trees~\citep{friedman2001greedy}. 

\textbf{ENET} represents Elastic-Net, which is a regularized regression method combining both L1 and L2 penalties of the lasso and ridge methods~\citep{zou2005regularization} and used in time series analysis~\citep{liu2010learning, bai2008forecasting}.

The third category of deep learning baselines includes:

\textbf{RETAIN} uses RNNs to respectively learn weights on input data for predicting~\citep{choi2016retain}.
It defines contribution coefficients on attentions to represent feature importance. 

\textbf{DUAL} is an encoder-decoder architecture using an encoder to learn attentions and feeding pre-weighted input data into a decoder for forecasting~\citep{qin2017dual}.
It uses temporal-wise variable attentions to reflect variable importance.


In ARIMAX, the orders of auto-regression and moving-average terms are set as the window size of the training data.
For RF and XGT, hyper-parameter tree depth and the number of iterations are chosen from range $[3, 10]$ and $[2, 200]$ via grid search. For XGT, L2 regularization is added by searching within $\{0.0001, 0.001, 0.01, 0.1, 1, 10\}$.
As for ENET, the coefficients for L2 and L1 penalties are selected from $\{0, 0.1, 0.3, 0.5, 0.7, 0.9, 1, 2\}$. 
For machine learning baselines, multi-variable input sequences are flattened into feature vectors. 

We implemented IMV-LSTM and deep learning baselines with Tensorflow.
We used Adam with the mini-batch size $64$~\citep{kingma2014adam}.
For the size of recurrent and dense layers in the baselines, we conduct grid search over $\{16, 32, 64, 128, 256, 512\}$. 
The size of IMV-LSTM layers is set by the number of neurons per variable selected from $\{10, 15, 20, 25\}$.
Dropout is selected in $\{0, 0.2, 0.5\}$.
Learning rate is searched in $\{0.0005, 0.001, 0.005, 0.01, 0.05 \}$.
L2 regularization is added with the coefficient chosen from $\{0.0001, 0.001, 0.01, 0.1, 1.0\}$.
We train each approach $5$ times and report average performance.
The window size (i.e. $T$) for PM2.5 and SML is set to 10 according to~\citep{qin2017dual}, while for PLANT it is 20 to test long dependency.

We consider two metrics to measure the prediction performance. RMSE is defined as $ \text{RMSE} = \sqrt[]{\sum_{k} (y_k - \hat{y}_k)^2/K} $. MAE is defined as $ \text{MAE} = \sum_{k}|y_k - \hat{y}_k|/K$.

\subsection{Prediction Performance}
We report the prediction errors in Table~\ref{tab:full}, each cell of which presents the average RMSE and MAE with standard errors.
In particular, IMV-LSTM family outperforms baselines by around $80\%$ at most.
Deep learning baselines mostly outperform other baselines.
Boosting method XGT presents comparable performance with deep learning baselines in PLANT and SML datasets.

\textbf{Insights.} For multi-variable data carrying different patterns, properly modeling individual variables and their interaction is important for the prediction performance.
IMV-Full keeps the variable interaction in the gate updating. 
IMV-Tensor maintains variable-wise hidden states independently and only captures their interaction via the mixture attention.
Experimentally, IMV-Full and IMV-Tensor present comparable performance, though mixture on independent variable-wise hidden states in IMV-Tensor leads to the best performance.
Note that IMV-Full and IMV-Tensor are of single network structure.
Instead of composite network architectures in baselines, mixture of well-maintained variable-wise hidden states in IMV-LSTM also improves the prediction performance and empowers the interpretability shown below.

\subsection{Interpretation}
In this part, we qualitatively analyze the meaningfulness of variable and temporal importance.
Fig.~\ref{fig:impt_var} and Fig.~\ref{fig:impt_temp} respectively show the variable and temporal importance values during the training under the best hyper-parameters. 
The importance values learned by IMV-Full and IMV-Tensor could be slightly different, because in IMV-Tensor the gate and memory update scheme evolve independently, thereby leading to different hidden states to IMV-Full. 
IMV-LSTM is easier to understand, compared to baseline RETAIN and DUAL, since they do not show in global level importance interpretation like in Fig.~\ref{fig:impt_var} and ~\ref{fig:impt_temp}.
\begin{table*}[htbp]
  \centering
  \caption{\small{RMSE and MAE with std. errors under top $50\%$ important variables}}
  \small
\resizebox{1.0\textwidth}{!}{%
  \begin{tabular}{|C{1.85cm}|c|c|c|}
    \hline
     {Dataset} &  PM2.5 & PLANT & SML \\
    \hline
    DUAL &      $ 25.09 \pm 0.04$, $15.59 \pm 0.08$ &  $171.30 \pm 0.17$, $154.15 \pm 0.20$ & $0.026 \pm 0.002$, $0.018 \pm 0.002$ \\
    RETAIN &    $ 49.25 \pm 0.11$, $34.03 \pm 0.24$ & $226.38 \pm 0.72$, $167.90 \pm 0.81$  & $0.060 \pm 0.001$, $0.044 \pm 0.004$ \\
    \hline
    IMV-Full-P &  $25.14 \pm 0.54$, $ 16.01 \pm 0.52 $  &  $165.04 \pm 0.08$, $129.09 \pm 0.09$  & $0.016 \pm 0.001$,  $0.013 \pm 0.0009$ \\
    IMV-Tensor-P & $24.84 \pm 0.43$, $ 15.57 \pm 0.61$  &  $161.98 \pm 0.11$, $131.17 \pm 0.12$  & $0.013 \pm 0.0008$, $0.009 \pm 0.0005$ \\
    \hline
    IMV-Full &    $24.32 \pm 0.32$,  $15.47 \pm 0.02$     & $162.14 \pm 0.10$, $128.83 \pm 0.12$  & $0.015 \pm 0.001$, $0.011 \pm 0.002 $ \\
    IMV-Tensor &  $\mathbf{24.12 \pm 0.03}$, $\mathbf{15.10 \pm 0.01}$   & $\mathbf{157.64 \pm 0.14}$, $\mathbf{128.16 \pm 0.13}$ & $\mathbf{0.007 \pm 0.0005}$, $\mathbf{0.006 \pm 0.0003}$ \\
    \hline
\end{tabular}%
}
\vspace{-10pt}
\label{tab:top50}
\end{table*}

\textbf{Variable importance.}
In Fig.~\ref{fig:impt_var}, top and bottom panels in each sub-fig show the variable importance values w.r.t. training epochs from IMV-Full and IMV-Tensor. Overall, variable importance values converge during the training and the ranking of variable importance is identified at the end of the training.
Variables with high importance values contribute more to the prediction of IMV-LSTM.

In Fig.~\ref{fig:impt_var}(a), for PM2.5 dataset, variables ``Wind speed'', ``Pressure'', ``Snow'', ``Rain'' are high ranked by IMV-LSTM. 
According to a recent work studying air pollution~\citep{liang2015assessing}, ``Dew Point'' and ``Pressure'' are both related to PM2.5 and they are also inter-correlated. 
One ``Pressure'' variable is enough to learn accurate forecasting and thus has the high importance value.
Strong wind can bring dry and fresh air.
``Snow'' and ``Rain'' amount are related to the air quality as well.
Variables important for IMV-LSTM are in line with the domain knowledge in \citep{liang2015assessing}.

Fig.~\ref{fig:impt_var}(b) shows that in PLANT dataset in addition to ``Irradiance'' and  ``Cloud cover'', ``Wind speed'', ``Humidity'' as well as ``Temperature'' are also high ranked and relatively used more in IMV-LSTM to provide accurate forecasting. 
As is discussed in \citep{mekhilef2012effect, ghazi2014effect}, humidity causes dust deposition and consequentially degradation in solar cell efficiency. Increased wind can move heat from the cell surface, which leads to better efficiency.

Fig.~\ref{fig:impt_var}(c) demonstrates that variables ``Humid. room'', ``CO2 room'', and ``Lighting room''
are relatively more important for IMV-LSTM (``Humid.'' represents humidity). 
As is suggested in~\citep{nguyen2014relationship, hoppe1993indoor}, humidity is correlated to the indoor temperature.

\textbf{Temporal importance.}
Fig.~\ref{fig:impt_temp} demonstrate the temporal importance values of each variable at the ending of the training.
The lighter the color, the more the corresponding data contributes to the prediction.

Specifically, in Fig.~\ref{fig:impt_temp}(a), short history of variables ``Snow'' and ``Wind speed'' contributes more to the prediction. PM2.5 itself has relatively long-term auto-correlation, i.e. around $5$ hours.   
Fig.~\ref{fig:impt_temp}(b) shows that recent data of aforementioned important variables ``Wind'' and ``Temperature'' are highly used for prediction, while ``Cloud-cover'' is long-term correlated to the target, i.e. around $13$ hours.
In Fig.~\ref{fig:impt_temp}(c) temporal importance values are mostly uniform, though ``Humid. dinning'' has short correlation, ``Outdoor temp.'' and ``Lighting dinning'' are relatively long-term correlated to the target.
\vspace{-0.17cm}
\begin{figure}[!ht]
    \centering
    \begin{subfigure}[]{0.45\textwidth}
        \centering
        \includegraphics[width=0.99\textwidth]{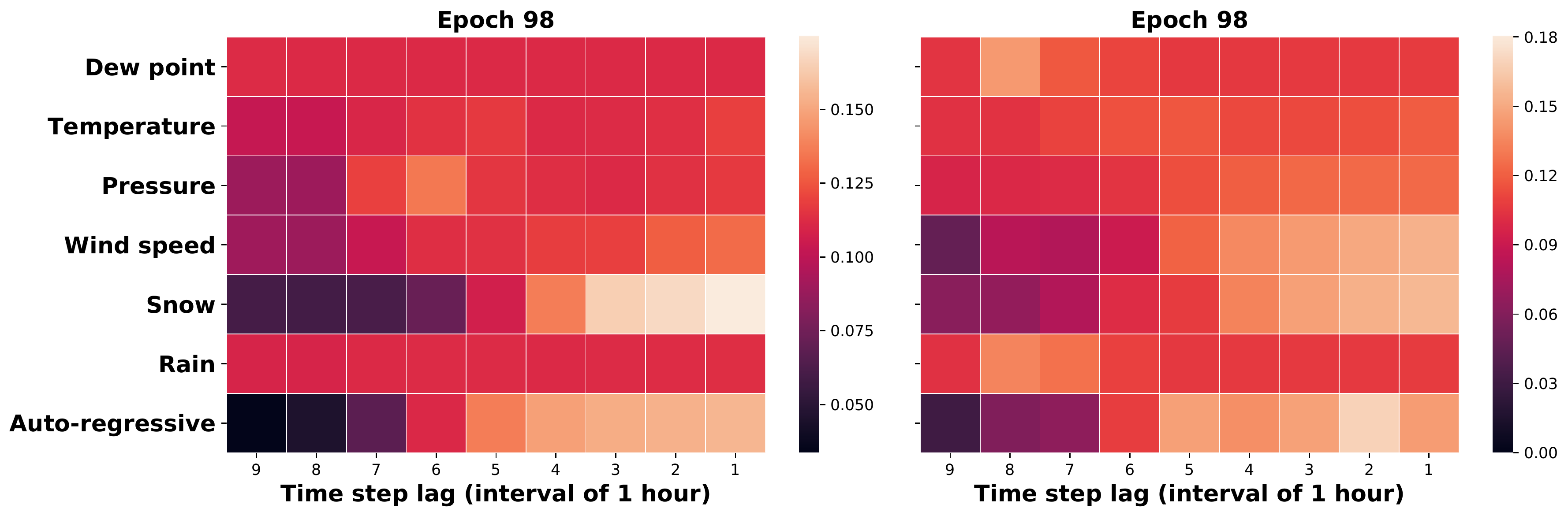}
        \caption{\small{PM2.5}}
    \end{subfigure}
    ~
    \begin{subfigure}[]{0.45\textwidth}
        \centering
        \includegraphics[width=0.99\textwidth]{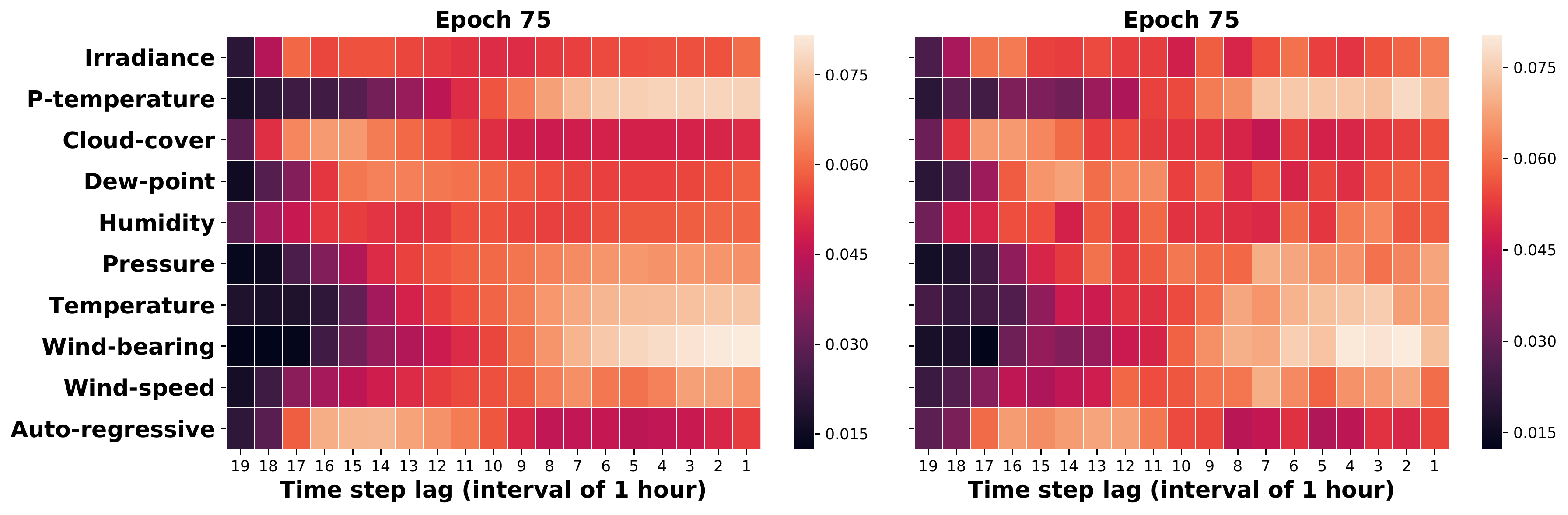}
        \caption{\small{PLANT}}
    \end{subfigure}
    ~
    \begin{subfigure}[]{0.45\textwidth}
        \centering
        \includegraphics[width=0.99\textwidth]{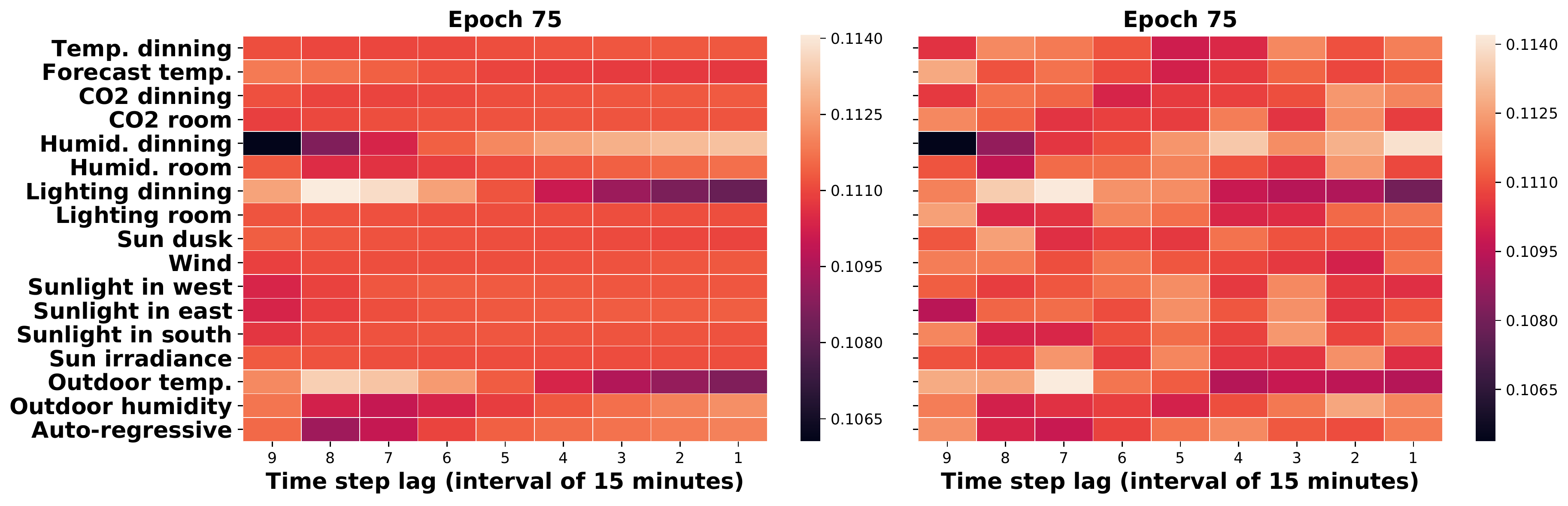}
        \caption{\small{SML}}
    \end{subfigure}
    ~
    \vspace{-10pt}
    \caption{\small Variable-wise temporal importance interpretation. Left and right panels respectively correspond to IMV-Full and IMV-Tensor. Different patterns of decays in the temporal importance of variables are observed, which brings the additional interpretation of the dynamics of each variable. (Best viewed in color)}
    \vspace{-15pt}
    \label{fig:impt_temp}
\end{figure}

\subsection{Variable Selection}
In this group of experiments, we quantitatively evaluate the efficacy of variable importance through the lens of prediction tasks.
We focus on IMV-LSTM family and RNN baselines, i.e. DUAL and RETAIN.

Specifically, for each approach, we first rank variables respectively according to the variable importance in IMV-LSTM, variable attention in DUAL and contribution coefficients in RETAIN.
Meanwhile, we add one more group of baselines denoted by IMV-Full-P and IMV-Tensor-P. 
The label ``-P'' represents that the Pearson correlation is used to rank the variables with the highest (absolute) correlation values to the target and the selected data is fed to IMV-LSTM.

Then we rebuild datasets only consisting of top $50\%$ ranked variables by respective methods, retrain each model with these new datasets and obtain the errors in Table~\ref{tab:top50}.

\textbf{Insights.}
Ideally, effective variable selection enables the corresponding retrained models to have comparable errors in comparison to their counterparts trained on full data in Table~\ref{tab:full}.
IMV-Full and IMV-Tensor present comparable and even lower errors in Table~\ref{tab:top50}, while DUAL and RETAIN have higher errors mostly. 
Pearson correlation measures linear relation. Selecting variables based on it neglects non-linear correlation and is not suitable for LSTM to attain the best performance.
An additional advantage of variable selection is the training efficiency, e.g. training time of each epoch in IMV-Tensor is reduced from $\sim$16 to $\sim$11 sec.

%% file: conclusion.tex
\section{Conclusion and Discussion}
In this paper, we explore the internal structures of LSTMs for interpretable prediction on multi-variable time series. 
Based on the hidden state matrix, we present two realizations i.e. IMV-Full and IMV-Tensor, which enable to infer and quantify variable importance and variable-wise temporal importance w.r.t. the target.
Extensive experiments provide insights into achieving superior prediction performance and importance interpretation for LSTM. 

Regarding high order effect, e.g. variable interaction in data, it can be captured by 
adding additional rows into hidden state matrices and additional elements into importance vectors accordingly. 
This will be the future work.

%% file: appendix.tex
\section{Appendix}

\subsection{Interpretable Multi-Variable LSTM}

\textbf{Proof of Lemma 3.3}
\begin{proof}
For simplicity, we ignore the data instance index $m$ in the following proof.

The log-likelihood of the target conditional on input variables is defined as:
\begin{equation}\label{eq:em}
\begin{split}
& \log p(y_{T+1} \, | \mathbf{X}_{T}) = \log \sum_{n=1}^{N} p(y_{T+1}, z_{T+1} = n | \mathbf{X}_{T} ) \\
& \geq \sum_{n=1}^{N} q^n \log p(y_{T+1} | z_{T+1} = n, \mathbf{X}_{T}) \Pr(z_{T+1} = n, \mathbf{X}_{T}) - q^n \log q^n \\
& = \sum_{n=1}^{N} q^n [\log p(y_{T+1} | z_{T+1} = n, \mathbf{X}_{T}) + \log \Pr(z_{T+1} = n, \mathbf{X}_{T})] + q^n \log q^n - 2 q^n \log q^n
\end{split}
\end{equation}

Based on Gibbs inequality, we can have 
\begin{equation}\label{eq:gibbs}
\sum_{n=1}^{N} q^n \log q^n \geq \sum_{n=1}^{N} q^n \log \Pr(z_{T+1} = n | \mathbf{I})
\end{equation}

Since $\mathbf{I} \in R^{N}_{\geq 0}$, $\sum_{n=1}^N \text{I}_{n} = 1$, it can parameterize a categorical distribution on $z_{T+1}$. 

Then introducing Eq.~\ref{eq:gibbs} into Eq.~\ref{eq:em}, we can obtain
\begin{equation}\label{eq:bound}
\begin{split}
\log p(y_{T+1} \, | \mathbf{X}_{T}) & \geq \sum_{n=1}^{N} q^n [\log p(y_{T+1} | z_{T+1} = n, \mathbf{X}_{T}) + \log \Pr(z_{T+1} = n, \mathbf{X}_{T})] + q^n \log \Pr(z_{T+1} = n | \mathbf{I}) - 2 q^n \log q^n \\
& \approx \mathbb{E}_{q^n}[ \, \log p( y_{T+1, } \, | \, z_{T+1, } = n, \mathbf{h}_T^n \oplus \mathbf{g}^n ) ] + \mathbb{E}_{q^n}[ \, \log \Pr(z_{T+1, } = n \, | \, \mathbf{h}_T^1 \oplus \mathbf{g}^1, \cdots, \mathbf{h}_T^N \oplus \mathbf{g}^N )] \\
&\quad + \mathbb{E}_{q^n}[ \, \log \Pr(z_{T+1, } = n \, | \,  \mathbf{I} )] - 2 q^n \log q^n
\end{split}
\end{equation}
During the EM process, after the E-step, $2 q^n \log q^n$ will be a constant and is not involved in the optimization process. In the M-step, minimizing the negative log-likelihood amounts to minimize the loss function as follows:
\begin{equation*}
\textstyle
\begin{split}
\mathcal{L}(\Theta, \mathbf{I} ) = & - \sum_{m=1}^{M}  \mathbb{E}_{q_m^n}[ \, \log p( y_{T+1, m} \, | \, z_{T+1, m} = n, \mathbf{h}_T^n \oplus \mathbf{g}^n ) ] \\
&\quad - \mathbb{E}_{q_m^n}[ \,\log \Pr(z_{T+1, m} = n \, | \, \mathbf{h}_T^1 \oplus \mathbf{g}^1, \cdots, \mathbf{h}_T^N \oplus \mathbf{g}^N )] \\
&\quad - \mathbb{E}_{q_m^n}[ \, \log \Pr(z_{T+1, m} = n \, | \,  \mathbf{I} )
]
\end{split}
\end{equation*}

\end{proof}

\subsection{Experiments}

In this part, we provide complementary experiment results as well as the insights from the results. 

\textbf{NASDAQ} is the dataset from \citep{qin2017dual}. It contains 81 major corporations under NASDAQ 100, as exogenous time series. 
The index value of the NASDAQ 100 is the target series. 
The frequency of the data collection is minute-by-minute. 
The first 35,100, the following 2,730 and the last 2,730 data points are respectively used as the training, validation and test sets.

\subsubsection{Prediction performance analysis} \label{subsubsec:performance_analysis}

\begin{table}[htbp]
  \centering
  \caption{\small{RMSE and MAE with std. errors}}
  \small
 \resizebox{0.5\textwidth}{!}{%
   \begin{tabular}{|c|c|c|c|}
    \hline
     {Dataset} &  NASDAQ   \\
    \hline
    STRX &     $0.41 \pm 0.01$, $0.35 \pm 0.02$   \\
    ARIMAX &   $0.34 \pm 0.02$, $0.23 \pm 0.03$   \\
    \hline
    RF &       $0.31 \pm 0.02$, $0.27 \pm 0.03$    \\
    XGT &      $0.28 \pm 0.01$, $0.23 \pm 0.02$    \\
    ENET &     $0.31 \pm 0.03$, $0.21 \pm 0.01$   \\
    \hline
    DUAL &     $0.31 \pm 0.003$, $0.21 \pm 0.002$  \\
    RETAIN &   $0.12 \pm 0.07$ , $0.11 \pm 0.06$  \\
    \hline
    IMV-Full & $0.27 \pm 0.01$, $0.23 \pm 0.01$  \\
    IMV-Tensor &  $\mathbf{0.09 \pm 0.005}$, $\mathbf{0.07 \pm 0.004}$  \\
    \hline
\end{tabular}%
}
\label{tab:appendix_full}
\end{table}

\subsubsection{Model Interpretation}

\begin{figure*}[htbp!]
    \centering
    \begin{subfigure}[t]{0.9\textwidth}
        \centering
        \includegraphics[width=0.99\textwidth]{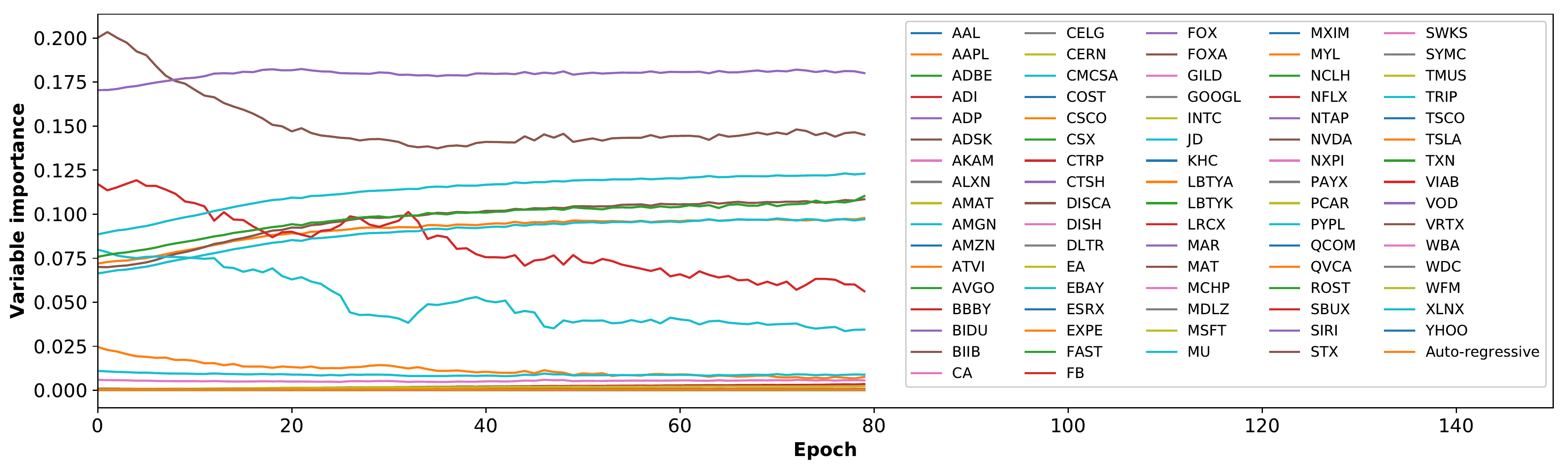}
        \caption{Variable importance w.r.t. epochs.}
    \end{subfigure}
    \begin{subfigure}[t]{0.9\textwidth}
        \centering
        \includegraphics[width=0.99\textwidth]{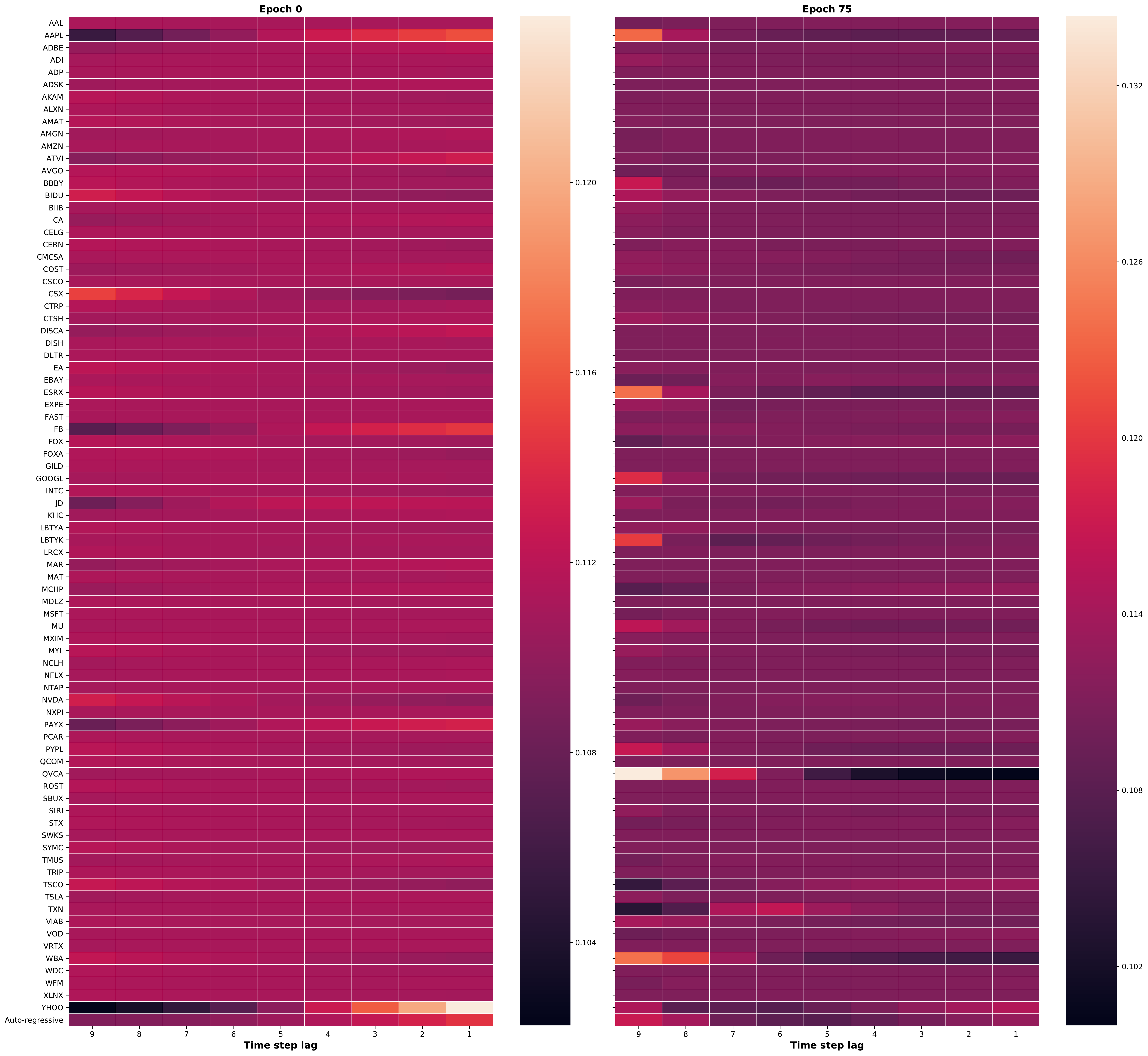}
        \caption{Variable-wise temporal importance at different epochs.}
    \end{subfigure}
    \caption{IMV-LSTM on NASDAQ dataset. (Best viewed in color)}
    \label{fig:imv_full_nasdaq}
\end{figure*}

In the following Table~\ref{tab:fullrank_nasdaq_imv}, \ref{tab:fullrank_nasdaq_dual_retain}, and \ref{tab:fullrank_sml_plant}, we list the full ranking of variables of the datasets by each approach. Variables associated with the importance or attention values are ranked in decreasing order. 

\begin{table*}[tbhp!]
  \centering
  \caption{Variable importance ranking by IMV-LSTM on NASDAQ dataset.}
  \small
\resizebox{1.0\textwidth}{!}{
  \begin{tabular}{|m{1.2cm}|l|p{10cm}|}
    \hline
    Dataset & Method & Rank of variables according to importance  \\
    \hline
    \\[-1em]
    \multirow{3}{*}{NASDAQ} & IMV-LSTM &  'ADSK', 0.00023858716, 'PAYX', 0.00023869322, 'AAL', 0.00023993119, 'MYL', 0.00024015515, 'CA', 0.00024144033], 'FOX', 0.00024341498, 'EA', 0.00024963205], 'BIDU', 0.00025009923, 'MCHP', 0.00025015706, 'QVCA', 0.00025018162, 'NVDA', 0.00025088928, 'WBA', 0.00025147066, 'LRCX', 0.00025165512, 'TSCO', 0.00025247637, 'CTSH', 0.00025284023, 'CSX', 0.00025417344, 'COST', 0.00025498777, 'BIIB', 0.00025547648, 'LBTYA', 0.00025680827, 'SIRI', 0.00025686354, 'ADBE', 0.00025687047, 'MDLZ', 0.00025788756, 'LBTYK', 0.00025885308, 'INTC', 0.00025894548, 'TSLA', 0.0002592771, 'WFM', 0.00025941888, 'SBUX', 0.00025953245, 'AVGO', 0.00026012328], 'CTRP', 0.00026024296, 'AMZN', 0.00026168497, 'ALXN', 0.00026173133, 'AMGN', 0.0002617908, 'GILD', 0.0002619058, 'VOD', 0.00026195042, 'ROST', 0.00026237246, 'NXPI', 0.0002624988, 'KHC', 0.0002625609, 'ADP', 0.0002626155, 'WDC', 0.00026269013, 'QCOM', 0.00026288, 'TMUS', 0.00026333777, 'AMAT', 0.00026334616, 'AKAM', 0.00026453246, 'PCAR', 0.00026510606, 'CERN', 0.00026535543, 'VRTX', 0.00026579297, 'MU', 0.00026719182, 'MAR', 0.00026789604, 'TXN', 0.00026821258, 'GOOGL', 0.0002684545, 'ESRX', 0.00026995668, 'ATVI', 0.0002703378, 'STX', 0.0002708045, 'FAST', 0.00027182887, 'EXPE', 0.0002747627, 'CELG', 0.00027897576, 'PYPL', 0.00027971127, 'MXIM', 0.0002802631, 'NFLX', 0.00028330996, 'BBBY', 0.00028975168, 'SYMC', 0.0002932911, 'CMCSA', 0.00031882498, 'SWKS', 0.00034903747, 'DLTR', 0.0004099159, 'YHOO', 0.0004359138, 'VIAB', 0.00046212596, 'Auto-regressive', 0.0004718905, 'MAT', 0.0008193875, 'MSFT', 0.002350653, 'ADI', 0.0035426863, 'DISH', 0.0056709386, 'AAPL', 0.007597621, 'EBAY', 0.008922806, 'JD', 0.03449823, 'FB', 0.056254942, 'XLNX', 0.09711476, 'CSCO', 0.09782402, 'DISCA', 0.108503476, 'NCLH', 0.11029968, u'TRIP', 0.12302372, 'FOXA', 0.14510903, 'NTAP', 0.18010232 \\
    \hline
\end{tabular}
}
\label{tab:fullrank_nasdaq_imv}
\end{table*}

\begin{table*}[tbhp!]
  \centering
  \caption{Variable importance ranking by DUAL and RETAIN methods on NASDAQ dataset.}
  \small
\resizebox{1.0\textwidth}{!}{
  \begin{tabular}{|m{1.2cm}|l|p{10cm}|}
    \hline
    Dataset & Method & Rank of variables according to importance  \\
    \hline
    \\[-1em]
    \multirow{3}{*}{NASDAQ} 
    & DUAL &   'NXPI', 0.003557, 'QCOM', 0.003564, 'FOX', 0.003566, 'NTAP', 0.003566, 'CELG', 0.003566, 'FOXA', 0.003567, 'PAYX', 0.003567, 'AAPL', 0.003567, 'WFM', 0.003567, 'ADSK', 0.003567, 'SBUX', 0.003567, 'STX', 0.003567, 'AKAM', 0.003567, 'DISH', 0.003567, 'AVGO', 0.003567, 'XLNX', 0.003567, 'AAL', 0.003567, 'FAST', 0.003567, 'TMUS', 0.003567, 'LRCX', 0.003567, 'NCLH', 0.003567, 'MCHP', 0.003567, 'MSFT', 0.003567, 'MU', 0.003567, 'NFLX', 0.003567, 'NVDA', 0.003567, 'PCAR', 0.003567, 'SIRI', 0.003567, 'MAR', 0.003567, 'TXN', 0.003567, 'ROST', 0.003567, 'CMCSA', 0.003567, 'ADI', 0.003567, 'ADP', 0.003567, 'DISCA', 0.003567, 'AMAT', 0.003567, 'WDC', 0.003567, 'CSX', 0.003567, 'WBA', 0.003567, 'GOOGL', 0.003622, 'COST', 0.003678, 'INTC', 0.003712, 'CTSH', 0.003908, 'BBBY', 0.004027, 'TRIP', 0.004881, 'MAT', 0.004956, 'ATVI', 0.005121, 'LBTYK', 0.00523, 'CERN', 0.00524, 'CTRP', 0.005283, 'ALXN', 0.00536, 'VOD', 0.005369, 'VRTX', 0.005433, 'LBTYA', 0.005445, 'MXIM', 0.00554, 'BIIB', 0.005554, 'EBAY', 0.005555, 'BIDU', 0.005605, 'FB', 0.005654, 'VIAB', 0.005685, 'GILD', 0.005695, 'AMGN', 0.005716, 'MYL', 0.005737, 'YHOO', 0.006166, 'KHC', 0.006555, 'AMZN', 0.006605, 'CSCO', 0.007836, 'ESRX', 0.010614, 'SWKS', 0.012777, 'MDLZ', 0.017898, 'CA', 0.02198, 'EXPE', 0.024373, 'QVCA', 0.026462, 'EA', 0.027808, 'TSLA', 0.043082, 'ADBE', 0.043829, 'JD', 0.071079, 'SYMC', 0.081596, 'PYPL', 0.087612, 'DLTR', 0.119737, 'TSCO', 0.122887\\
    & RETAIN & 'DLTR', 0.000866, 'QVCA', 0.001128, 'TSLA', 0.00119, 'PYPL', 0.00128, 'EA', 0.001439, 'EXPE', 0.001502, 'CA', 0.001713, 'TSCO', 0.001737, 'SYMC', 0.002334, 'ADBE', 0.00252, 'JD', 0.002607, 'AMZN', 0.003367, 'CSCO', 0.003543, 'KHC', 0.003996, 'CTSH', 0.004695, 'NXPI', 0.004865, 'EBAY', 0.004963, 'SWKS', 0.005011, 'MXIM', 0.005135, 'MYL', 0.005541, 'COST', 0.006052, 'BIDU', 0.006534, 'GOOGL', 0.006906, 'INTC', 0.007153, 'GILD', 0.007212, 'ESRX', 0.007512, 'NTAP', 0.007695, 'QCOM', 0.008037, 'CELG', 0.008168, 'MDLZ', 0.008829, 'AMGN', 0.008998, 'FOX', 0.009943, 'VIAB', 0.010123, 'AAPL', 0.010157, 'FB', 0.010359, 'YHOO', 0.010744, 'PAYX', 0.010899, 'BBBY', 0.01117, 'AKAM', 0.012054, 'BIIB', 0.012069, 'NFLX', 0.012266, 'ADSK', 0.012319, 'DISH', 0.012338, 'LBTYA', 0.012697, 'FOXA', 0.01282, 'MCHP', 0.012833, 'WFM', 0.012869, 'STX', 0.012887, 'VRTX', 0.013318, 'SBUX', 0.013458, 'VOD', 0.013798, 'ALXN', 0.013878, 'CTRP', 0.013963, 'SIRI', 0.01475, 'CERN', 0.014777, 'LBTYK', 0.014799, 'ATVI', 0.015651, 'AVGO', 0.016382, 'CMCSA', 0.016531, 'TXN', 0.016977, 'LRCX', 0.017131, 'AMAT', 0.017378, 'ROST', 0.017399, 'MU', 0.018045, 'TRIP', 0.018236, 'MAT', 0.018297, 'Auto-regressive', 0.018626, 'WDC', 0.019083, 'DISCA', 0.019233, 'FAST', 0.019392, 'CSX', 0.019734, 'WBA', 0.019984, 'AAL', 0.021188, 'ADI', 0.021215, 'NCLH', 0.022932, 'NVDA', 0.022994, 'TMUS', 0.024187, 'MSFT', 0.026354, 'ADP', 0.028515, 'MAR', 0.028783, 'PCAR', 0.029459, 'XLNX', 0.03248\\
    \hline
\end{tabular}
}
\label{tab:fullrank_nasdaq_dual_retain}
\end{table*}




\begin{table*}[tbhp!]
  \centering
  \caption{Variable importance ranking on PLANT and SML datasets.}
  \small
\resizebox{1.0\textwidth}{!}
{
  \begin{tabular}{|m{1.2cm}|l|p{10cm}|}
    \hline
    Dataset & Method & Rank of variables according to importance  \\
    \hline
    \\[-1em]
    \multirow{3}{*}{PLANT} & IMV-LSTM & 'Dew-point', 0.040899094, 'Wind-bearing', 0.04476319, 'Pressure', 0.06180005, 'P-temperature', 0.07244386, 'Auto-regressive', 0.1083069, 'Temperature', 0.11868146, 'Irradiance', 0.12043289, 'Humidity', 0.13192631, 'Cloud-cover', 0.14283147, 'Wind-speed', 0.15791483 \\
    & DUAL &  
    'Irradiance', 0.06128826, 'Dew-point', 0.066655099, 'Temperature', 0.071131147, 'Wind-speed', 0.094427079, 'Wind-bearing', 0.106529392, 'P-temperature', 0.115000054, 'Pressure', 0.115962856, 'Cloud cover', 0.144996881, 'Humidity', 0.224009201\\
    & RETAIN & 'Dewpoint', 0.031317, 'Temperature', 0.037989, 'Wind-bearing', 0.044226, 'Wind-speed', 0.052027, 'P-temperature', 0.053034, 'Cloud cover', 0.138427, 'Irradiance', 0.142899, 'Auto-regressive', 0.143269, 'Humidity', 0.172893, 'Pressure', 0.183919
    \\
    \\[-1em]
    \hline
    \\[-1em]
    \multirow{3}{*}{SML} & IMV-LSTM & 'Outdoor temp.', 0.008530081, 'Outdoor humidity', 0.0120737655, 'Sun irradiance', 0.012943255, 'CO2 dining', 0.01563413, 'Sunlight in south', 0.01569774, 'Sun dusk', 0.015769556, 'Wind', 0.015868865], 'Forecast temp.', 0.015990425, 'Sunlight in west', 0.01609429, 'Lighting dining', 0.016338758, 'Humid. dining', 0.016379833, 'Sunlight in east', 0.016386982, 'Auto-regressive', 0.016530316, 'Temp. dining', 0.01663947, 'Lighting room', 0.18322693, 'CO2 room', 0.26715645, 'Humid. room', 0.33873916 \\
     & DUAL & 'Humid. room', 0.059424, 'Humid. dining', 0.059656, 'Outdoor humidity', 0.059803, 'Temp. dining', 0.059878, 'Sun dusk', 0.060408, 'Sunlight in south', 0.061626, 'Wind', 0.061629, 'Sunlight in east', 0.062792, 'Lighting room', 0.063381, 'Forecast temp.', 0.063503, 'Sunlight in west', 0.063832, 'CO2 room', 0.064149, 'CO2 dining', 0.064383, 'Sun irradiance', 0.064703, 'Lighting dining', 0.0651, 'Outdoor temp.', 0.065733
     \\
    & RETAIN & 'Humid. dining', 0.012169, 'Humid. room', 0.014563, 'Sunlight in south', 0.018446, 'Lighting room', 0.018732, 'Outdoor humidity', 0.019388, 'Sunlight in west', 0.02219, 'Sunlight in east', 0.036744, 'CO2 room', 0.036864, 'CO2 dining', 0.037174, 'Sun dusk', 0.040011, 'Sun irradiance', 0.04075, 'Wind', 0.041191, 'Lighting dining', 0.054166, 'Forecast temp.', 0.133079, 'Outdoor temp.', 0.144314, 'Auto-regressive', 0.164673, 'Temp. dining', 0.165543
    \\
    \hline
\end{tabular}
}
\label{tab:fullrank_sml_plant}
\end{table*}

\subsection{Discussion}

In this part, we summarize the insights from the experiments.

\paragraph{Prediction performance} 
For multi-variable data, capturing individual variable's behaviors and their interaction is the key for both prediction and interpretation.  
Conventional hidden states in standard LSTMs consume the data from all input variables at each 
step, while our IMV-LSTM family decomposes the hidden states by defining variable data flows for each hidden state element.

In the experiments, IMV-Full and IMV-Tensor outperform baselines using the traditional hidden states.
Multi-variable data potentially carries different dynamics. Conventional hidden states mix the data of all input variables, thereby failing to explicitly capture individual dynamics.
In the multi-variable setting, these opaque hidden states are a burden to both prediction and interpretation.

On the contrary, IMV-Tensor models individual variables and then uses mixture attention to capture the variable interaction by variable-wise hidden states.
It achieves superior prediction performance and enables the interpretability on both temporal and variable levels. 

\paragraph{Effectiveness of importance values}
For LSTM networks on multi-variable data, importance values inherently learned by the network are more suitable for retaining useful variables for predicting. 

By choosing the variables based on the learned importance value, IMV-LSTM family mostly retains the prediction performance and presents lower prediction errors on two datasets. 
The importance value in IMV-LSTM is derived during the training and therefore it is able to effectively identify the variables used by IMV-LSTM to minimize the loss function, i.e. maximize the prediction accuracy.  

Pearson correlation variable selection leads to the quality loss in prediction performance, i.e. higher errors. 
Pearson correlation measures the linear correlation and pre-selecting variables based on it neglects the potential non-linear correlation in data indispensable for LSTMs to capture.

%% file: main.bbl
\begin{thebibliography}{62}
\providecommand{\natexlab}[1]{#1}
\providecommand{\url}[1]{\texttt{#1}}
\expandafter\ifx\csname urlstyle\endcsname\relax
  \providecommand{\doi}[1]{doi: #1}\else
  \providecommand{\doi}{doi: \begingroup \urlstyle{rm}\Url}\fi

\bibitem[Ancona et~al.(2018)Ancona, Ceolini, Oztireli, and
  Gross]{ancona2018towards}
Ancona, M., Ceolini, E., Oztireli, C., and Gross, M.
\newblock Towards better understanding of gradient-based attribution methods
  for deep neural networks.
\newblock In \emph{6th International Conference on Learning Representations
  (ICLR 2018)}, 2018.

\bibitem[Arras et~al.(2017)Arras, Montavon, M{\"u}ller, and
  Samek]{arras2017explaining}
Arras, L., Montavon, G., M{\"u}ller, K.-R., and Samek, W.
\newblock Explaining recurrent neural network predictions in sentiment
  analysis.
\newblock \emph{arXiv preprint arXiv:1706.07206}, 2017.

\bibitem[Bahdanau et~al.(2014)Bahdanau, Cho, and Bengio]{bahdanau2014neural}
Bahdanau, D., Cho, K., and Bengio, Y.
\newblock Neural machine translation by jointly learning to align and
  translate.
\newblock In \emph{International Conference on Learning Representations}, 2014.

\bibitem[Bai \& Ng(2008)Bai and Ng]{bai2008forecasting}
Bai, J. and Ng, S.
\newblock Forecasting economic time series using targeted predictors.
\newblock \emph{Journal of Econometrics}, 146\penalty0 (2):\penalty0 304--317,
  2008.

\bibitem[Bishop(1994)]{bishop1994mixture}
Bishop, C.~M.
\newblock Mixture density networks.
\newblock 1994.

\bibitem[Ceci et~al.(2017)Ceci, Corizzo, Fumarola, Malerba, and
  Rashkovska]{ceci2017predictive}
Ceci, M., Corizzo, R., Fumarola, F., Malerba, D., and Rashkovska, A.
\newblock Predictive modeling of pv energy production: How to set up the
  learning task for a better prediction?
\newblock \emph{IEEE Transactions on Industrial Informatics}, 13\penalty0
  (3):\penalty0 956--966, 2017.

\bibitem[Che et~al.(2016)Che, Purushotham, Khemani, and
  Liu]{che2016interpretable}
Che, Z., Purushotham, S., Khemani, R., and Liu, Y.
\newblock Interpretable deep models for icu outcome prediction.
\newblock In \emph{AMIA Annual Symposium Proceedings}, volume 2016, pp.\  371.
  American Medical Informatics Association, 2016.

\bibitem[Chen \& Guestrin(2016)Chen and Guestrin]{chen2016xgboost}
Chen, T. and Guestrin, C.
\newblock Xgboost: A scalable tree boosting system.
\newblock In \emph{SIGKDD}, pp.\  785--794. ACM, 2016.

\bibitem[Cheng et~al.(2006)Cheng, Tan, Gao, and Scripps]{cheng2006multistep}
Cheng, H., Tan, P.-N., Gao, J., and Scripps, J.
\newblock Multistep-ahead time series prediction.
\newblock In \emph{Pacific-Asia Conference on Knowledge Discovery and Data
  Mining}, pp.\  765--774. Springer, 2006.

\bibitem[Cho et~al.(2014)Cho, Van~Merri{\"e}nboer, Bahdanau, and
  Bengio]{cho2014properties}
Cho, K., Van~Merri{\"e}nboer, B., Bahdanau, D., and Bengio, Y.
\newblock On the properties of neural machine translation: Encoder-decoder
  approaches.
\newblock \emph{arXiv preprint arXiv:1409.1259}, 2014.

\bibitem[Choi et~al.(2016)Choi, Bahadori, Sun, Kulas, Schuetz, and
  Stewart]{choi2016retain}
Choi, E., Bahadori, M.~T., Sun, J., Kulas, J., Schuetz, A., and Stewart, W.
\newblock Retain: An interpretable predictive model for healthcare using
  reverse time attention mechanism.
\newblock In \emph{Advances in Neural Information Processing Systems}, pp.\
  3504--3512, 2016.

\bibitem[Choi et~al.(2018)Choi, Cho, and Bengio]{choi2018fine}
Choi, H., Cho, K., and Bengio, Y.
\newblock Fine-grained attention mechanism for neural machine translation.
\newblock \emph{Neurocomputing}, 284:\penalty0 171--176, 2018.

\bibitem[Chu et~al.(2018)Chu, Hu, Hu, Wang, and Pei]{chu2018exact}
Chu, L., Hu, X., Hu, J., Wang, L., and Pei, J.
\newblock Exact and consistent interpretation for piecewise linear neural
  networks: A closed form solution.
\newblock In \emph{Proceedings of the 24th ACM SIGKDD international conference
  on Knowledge discovery and data mining}, pp.\  1244--1253, New York, NY, USA,
  2018. ACM.

\bibitem[Cinar et~al.(2017)Cinar, Mirisaee, Goswami, Gaussier, A{\"\i}t-Bachir,
  and Strijov]{cinar2017position}
Cinar, Y.~G., Mirisaee, H., Goswami, P., Gaussier, E., A{\"\i}t-Bachir, A., and
  Strijov, V.
\newblock Position-based content attention for time series forecasting with
  sequence-to-sequence rnns.
\newblock In \emph{International Conference on Neural Information Processing},
  pp.\  533--544. Springer, 2017.

\bibitem[Do et~al.(2017)Do, Tran, and Venkatesh]{do2017matrix}
Do, K., Tran, T., and Venkatesh, S.
\newblock Matrix-centric neural networks.
\newblock \emph{arXiv preprint arXiv:1703.01454}, 2017.

\bibitem[Feng et~al.(2018)Feng, Williamson, Carone, and
  Simon]{feng2018nonparametric}
Feng, J., Williamson, B.~D., Carone, M., and Simon, N.
\newblock Nonparametric variable importance using an augmented neural network
  with multi-task learning.
\newblock In \emph{International Conference on Machine Learning}, pp.\
  1495--1504, 2018.

\bibitem[Foerster et~al.(2017)Foerster, Gilmer, Sohl-Dickstein, Chorowski, and
  Sussillo]{foerster2017input}
Foerster, J.~N., Gilmer, J., Sohl-Dickstein, J., Chorowski, J., and Sussillo,
  D.
\newblock Input switched affine networks: An rnn architecture designed for
  interpretability.
\newblock In \emph{International Conference on Machine Learning}, pp.\
  1136--1145, 2017.

\bibitem[Fox et~al.(2018)Fox, Ang, Jaiswal, Pop-Busui, and
  Wiens]{deepmulti2018}
Fox, I., Ang, L., Jaiswal, M., Pop-Busui, R., and Wiens, J.
\newblock Deep multi-output forecasting: Learning to accurately predict blood
  glucose trajectories.
\newblock In \emph{Proceedings of the 24th ACM SIGKDD International Conference
  on Knowledge Discovery \&\#38; Data Mining}, KDD '18, pp.\  1387--1395. ACM,
  2018.

\bibitem[Friedman(2001)]{friedman2001greedy}
Friedman, J.~H.
\newblock Greedy function approximation: a gradient boosting machine.
\newblock \emph{Annals of statistics}, pp.\  1189--1232, 2001.

\bibitem[Ghazi \& Ip(2014)Ghazi and Ip]{ghazi2014effect}
Ghazi, S. and Ip, K.
\newblock The effect of weather conditions on the efficiency of pv panels in
  the southeast of uk.
\newblock \emph{Renewable Energy}, 69:\penalty0 50--59, 2014.

\bibitem[Graves(2013)]{graves2013generating}
Graves, A.
\newblock Generating sequences with recurrent neural networks.
\newblock \emph{arXiv preprint arXiv:1308.0850}, 2013.

\bibitem[Guo et~al.(2016)Guo, Xu, Yao, Chen, Aberer, and Funaya]{guo2016robust}
Guo, T., Xu, Z., Yao, X., Chen, H., Aberer, K., and Funaya, K.
\newblock Robust online time series prediction with recurrent neural networks.
\newblock In \emph{2016 IEEE DSAA}, pp.\  816--825. IEEE, 2016.

\bibitem[Guo et~al.(2018)Guo, Lin, and Lu]{guo2018interpretable}
Guo, T., Lin, T., and Lu, Y.
\newblock An interpretable lstm neural network for autoregressive exogenous
  model.
\newblock In \emph{workshop track at International Conference on Learning
  Representations}, 2018.

\bibitem[He et~al.(2017)He, Gao, Xiao, Liu, He, and Barber]{he2017wider}
He, Z., Gao, S., Xiao, L., Liu, D., He, H., and Barber, D.
\newblock Wider and deeper, cheaper and faster: Tensorized lstms for sequence
  learning.
\newblock In \emph{Advances in Neural Information Processing Systems}, pp.\
  1--11, 2017.

\bibitem[Hochreiter \& Schmidhuber(1997)Hochreiter and
  Schmidhuber]{hochreiter1997long}
Hochreiter, S. and Schmidhuber, J.
\newblock Long short-term memory.
\newblock \emph{Neural computation}, 9\penalty0 (8):\penalty0 1735--1780, 1997.

\bibitem[H{\"o}ppe(1993)]{hoppe1993indoor}
H{\"o}ppe, P.
\newblock Indoor climate.
\newblock \emph{Experientia}, 49\penalty0 (9):\penalty0 775--779, 1993.

\bibitem[Hu et~al.(2018)Hu, Liu, Bian, Liu, and Liu]{hu2018listening}
Hu, Z., Liu, W., Bian, J., Liu, X., and Liu, T.-Y.
\newblock Listening to chaotic whispers: A deep learning framework for
  news-oriented stock trend prediction.
\newblock In \emph{Proceedings of the Eleventh ACM International Conference on
  Web Search and Data Mining}, pp.\  261--269. ACM, 2018.

\bibitem[Hyndman \& Athanasopoulos(2014)Hyndman and
  Athanasopoulos]{hyndman2014forecasting}
Hyndman, R.~J. and Athanasopoulos, G.
\newblock \emph{Forecasting: principles and practice}.
\newblock OTexts, 2014.

\bibitem[Ke et~al.(2018)Ke, Zolna, Sordoni, Lin, Trischler, Bengio, Pineau,
  Charlin, and Pal]{ke2018focused}
Ke, N.~R., Zolna, K., Sordoni, A., Lin, Z., Trischler, A., Bengio, Y., Pineau,
  J., Charlin, L., and Pal, C.
\newblock Focused hierarchical rnns for conditional sequence processing.
\newblock \emph{arXiv preprint arXiv:1806.04342}, 2018.

\bibitem[Kingma \& Ba(2014)Kingma and Ba]{kingma2014adam}
Kingma, D.~P. and Ba, J.
\newblock Adam: A method for stochastic optimization.
\newblock \emph{arXiv preprint arXiv:1412.6980}, 2014.

\bibitem[Kirchg{\"a}ssner et~al.(2012)Kirchg{\"a}ssner, Wolters, and
  Hassler]{kirchgassner2012introduction}
Kirchg{\"a}ssner, G., Wolters, J., and Hassler, U.
\newblock \emph{Introduction to modern time series analysis}.
\newblock Springer Science \& Business Media, 2012.

\bibitem[Koutnik et~al.(2014)Koutnik, Greff, Gomez, and
  Schmidhuber]{koutnik2014clockwork}
Koutnik, J., Greff, K., Gomez, F., and Schmidhuber, J.
\newblock A clockwork rnn.
\newblock In \emph{International Conference on Machine Learning}, pp.\
  1863--1871, 2014.

\bibitem[Kuchaiev \& Ginsburg(2017)Kuchaiev and
  Ginsburg]{kuchaiev2017factorization}
Kuchaiev, O. and Ginsburg, B.
\newblock Factorization tricks for lstm networks.
\newblock \emph{arXiv preprint arXiv:1703.10722}, 2017.

\bibitem[Lai et~al.(2017)Lai, Chang, Yang, and Liu]{lai2017modeling}
Lai, G., Chang, W.-C., Yang, Y., and Liu, H.
\newblock Modeling long-and short-term temporal patterns with deep neural
  networks.
\newblock \emph{arXiv preprint arXiv:1703.07015}, 2017.

\bibitem[Liang et~al.(2015)Liang, Zou, Guo, Li, Zhang, Zhang, Huang, and
  Chen]{liang2015assessing}
Liang, X., Zou, T., Guo, B., Li, S., Zhang, H., Zhang, S., Huang, H., and Chen,
  S.~X.
\newblock Assessing beijing's pm2. 5 pollution: severity, weather impact, apec
  and winter heating.
\newblock In \emph{Proc. R. Soc. A}, volume 471, pp.\  20150257. The Royal
  Society, 2015.

\bibitem[Liaw et~al.(2002)Liaw, Wiener, et~al.]{liaw2002classification}
Liaw, A., Wiener, M., et~al.
\newblock Classification and regression by randomforest.
\newblock \emph{R news}, 2\penalty0 (3):\penalty0 18--22, 2002.

\bibitem[Lin et~al.(2017)Lin, Guo, and Aberer]{lin2017hybrid}
Lin, T., Guo, T., and Aberer, K.
\newblock Hybrid neural networks for learning the trend in time series.
\newblock In \emph{Proceedings of the Twenty-Sixth International Joint
  Conference on Artificial Intelligence, IJCAI-17}, pp.\  2273--2279, 2017.

\bibitem[Lipton(2016)]{lipton2016mythos}
Lipton, Z.~C.
\newblock The mythos of model interpretability.
\newblock \emph{arXiv preprint arXiv:1606.03490}, 2016.

\bibitem[Lipton et~al.(2015)Lipton, Kale, Elkan, and
  Wetzell]{lipton2015learning}
Lipton, Z.~C., Kale, D.~C., Elkan, C., and Wetzell, R.
\newblock Learning to diagnose with lstm recurrent neural networks.
\newblock \emph{arXiv preprint arXiv:1511.03677}, 2015.

\bibitem[Liu et~al.(2010)Liu, Niculescu-Mizil, Lozano, and Lu]{liu2010learning}
Liu, Y., Niculescu-Mizil, A., Lozano, A.~C., and Lu, Y.
\newblock Learning temporal causal graphs for relational time-series analysis.
\newblock In \emph{ICML}, pp.\  687--694, 2010.

\bibitem[Lundberg \& Lee(2017)Lundberg and Lee]{lundberg2017unified}
Lundberg, S.~M. and Lee, S.-I.
\newblock A unified approach to interpreting model predictions.
\newblock In \emph{Advances in Neural Information Processing Systems}, pp.\
  4765--4774, 2017.

\bibitem[Mekhilef et~al.(2012)Mekhilef, Saidur, and
  Kamalisarvestani]{mekhilef2012effect}
Mekhilef, S., Saidur, R., and Kamalisarvestani, M.
\newblock Effect of dust, humidity and air velocity on efficiency of
  photovoltaic cells.
\newblock \emph{Renewable and sustainable energy reviews}, 16\penalty0
  (5):\penalty0 2920--2925, 2012.

\bibitem[Montavon et~al.(2018)Montavon, Samek, and
  M{\"u}ller]{montavon2018methods}
Montavon, G., Samek, W., and M{\"u}ller, K.-R.
\newblock Methods for interpreting and understanding deep neural networks.
\newblock \emph{Digital Signal Processing}, 73:\penalty0 1--15, 2018.

\bibitem[Murdoch \& Szlam(2017)Murdoch and Szlam]{murdoch2017automatic}
Murdoch, W.~J. and Szlam, A.
\newblock Automatic rule extraction from long short term memory networks.
\newblock \emph{International Conference on Learning Representations}, 2017.

\bibitem[Murdoch et~al.(2018)Murdoch, Liu, and Yu]{murdoch2018beyond}
Murdoch, W.~J., Liu, P.~J., and Yu, B.
\newblock Beyond word importance: Contextual decomposition to extract
  interactions from lstms.
\newblock \emph{arXiv preprint arXiv:1801.05453}, 2018.

\bibitem[Neil et~al.(2016)Neil, Pfeiffer, and Liu]{neil2016phased}
Neil, D., Pfeiffer, M., and Liu, S.-C.
\newblock Phased lstm: Accelerating recurrent network training for long or
  event-based sequences.
\newblock In \emph{Advances in Neural Information Processing Systems}, pp.\
  3882--3890, 2016.

\bibitem[Nguyen et~al.(2014)Nguyen, Schwartz, and
  Dockery]{nguyen2014relationship}
Nguyen, J.~L., Schwartz, J., and Dockery, D.~W.
\newblock The relationship between indoor and outdoor temperature, apparent
  temperature, relative humidity, and absolute humidity.
\newblock \emph{Indoor air}, 24\penalty0 (1):\penalty0 103--112, 2014.

\bibitem[Novikov et~al.(2015)Novikov, Podoprikhin, Osokin, and
  Vetrov]{novikov2015tensorizing}
Novikov, A., Podoprikhin, D., Osokin, A., and Vetrov, D.~P.
\newblock Tensorizing neural networks.
\newblock In \emph{Advances in Neural Information Processing Systems}, pp.\
  442--450, 2015.

\bibitem[Patel et~al.(2015)Patel, Shah, Thakkar, and
  Kotecha]{patel2015predicting}
Patel, J., Shah, S., Thakkar, P., and Kotecha, K.
\newblock Predicting stock and stock price index movement using trend
  deterministic data preparation and machine learning techniques.
\newblock \emph{Expert Systems with Applications}, 42\penalty0 (1):\penalty0
  259--268, 2015.

\bibitem[Qin et~al.(2017)Qin, Song, Cheng, Cheng, Jiang, and
  Cottrell]{qin2017dual}
Qin, Y., Song, D., Cheng, H., Cheng, W., Jiang, G., and Cottrell, G.~W.
\newblock A dual-stage attention-based recurrent neural network for time series
  prediction.
\newblock In \emph{Proceedings of the 26th International Joint Conference on
  Artificial Intelligence}, IJCAI'17, pp.\  2627--2633. AAAI Press, 2017.

\bibitem[Radinsky et~al.(2012)Radinsky, Svore, Dumais, Teevan, Bocharov, and
  Horvitz]{radinsky2012modeling}
Radinsky, K., Svore, K., Dumais, S., Teevan, J., Bocharov, A., and Horvitz, E.
\newblock Modeling and predicting behavioral dynamics on the web.
\newblock In \emph{WWW}, pp.\  599--608. ACM, 2012.

\bibitem[Ribeiro et~al.(2018)Ribeiro, Singh, and Guestrin]{ribeiro2018anchors}
Ribeiro, M.~T., Singh, S., and Guestrin, C.
\newblock Anchors: High-precision model-agnostic explanations.
\newblock In \emph{AAAI Conference on Artificial Intelligence}, 2018.

\bibitem[Riemer et~al.(2016)Riemer, Vempaty, Calmon, Heath, Hull, and
  Khabiri]{riemer2016correcting}
Riemer, M., Vempaty, A., Calmon, F., Heath, F., Hull, R., and Khabiri, E.
\newblock Correcting forecasts with multifactor neural attention.
\newblock In \emph{International Conference on Machine Learning}, pp.\
  3010--3019, 2016.

\bibitem[Scott \& Varian(2014)Scott and Varian]{scott2014predicting}
Scott, S.~L. and Varian, H.~R.
\newblock Predicting the present with bayesian structural time series.
\newblock \emph{International Journal of Mathematical Modelling and Numerical
  Optimisation}, 5\penalty0 (1-2):\penalty0 4--23, 2014.

\bibitem[Shrikumar et~al.(2017)Shrikumar, Greenside, and
  Kundaje]{shrikumar2017learning}
Shrikumar, A., Greenside, P., and Kundaje, A.
\newblock Learning important features through propagating activation
  differences.
\newblock \emph{arXiv preprint arXiv:1704.02685}, 2017.

\bibitem[Sutskever et~al.(2014)Sutskever, Vinyals, and
  Le]{sutskever2014sequence}
Sutskever, I., Vinyals, O., and Le, Q.~V.
\newblock Sequence to sequence learning with neural networks.
\newblock In \emph{Advances in neural information processing systems}, pp.\
  3104--3112, 2014.

\bibitem[Vinyals et~al.(2015)Vinyals, Fortunato, and
  Jaitly]{vinyals2015pointer}
Vinyals, O., Fortunato, M., and Jaitly, N.
\newblock Pointer networks.
\newblock In \emph{Advances in Neural Information Processing Systems}, pp.\
  2692--2700, 2015.

\bibitem[Wang et~al.(2018)Wang, Wang, Li, and Wu]{wavelet}
Wang, J., Wang, Z., Li, J., and Wu, J.
\newblock Multilevel wavelet decomposition network for interpretable time
  series analysis.
\newblock In \emph{Proceedings of the 24th ACM SIGKDD International Conference
  on Knowledge Discovery}, pp.\  2437--2446, New York, NY, USA, 2018. ACM.

\bibitem[Xu et~al.(2018)Xu, Biswal, Deshpande, Maher, and Sun]{xu2018raim}
Xu, Y., Biswal, S., Deshpande, S.~R., Maher, K.~O., and Sun, J.
\newblock Raim: Recurrent attentive and intensive model of multimodal patient
  monitoring data.
\newblock In \emph{Proceedings of the 24th ACM SIGKDD International Conference
  on Knowledge Discovery \& Data Mining}, pp.\  2565--2573. ACM, 2018.

\bibitem[Zhang et~al.(2017)Zhang, Aggarwal, and Qi]{zhang2017stock}
Zhang, L., Aggarwal, C., and Qi, G.-J.
\newblock Stock price prediction via discovering multi-frequency trading
  patterns.
\newblock In \emph{Proceedings of the 23rd ACM SIGKDD International Conference
  on Knowledge Discovery and Data Mining}, pp.\  2141--2149. ACM, 2017.

\bibitem[Zong et~al.(2018)Zong, Song, Min, Cheng, Lumezanu, Cho, and
  Chen]{zong2018deep}
Zong, B., Song, Q., Min, M.~R., Cheng, W., Lumezanu, C., Cho, D., and Chen, H.
\newblock Deep autoencoding gaussian mixture model for unsupervised anomaly
  detection.
\newblock In \emph{International Conference on Learning Representations}, 2018.

\bibitem[Zou \& Hastie(2005)Zou and Hastie]{zou2005regularization}
Zou, H. and Hastie, T.
\newblock Regularization and variable selection via the elastic net.
\newblock \emph{Journal of the Royal Statistical Society: Series B (Statistical
  Methodology)}, 67\penalty0 (2):\penalty0 301--320, 2005.

\end{thebibliography}
